\documentclass[10pt,twocolumn]{IEEEtran}

\usepackage{amsmath,amssymb,amsthm}
\usepackage[utf8x]{inputenc}
\usepackage{graphicx}
\usepackage{indentfirst}
\usepackage{cmap}
\usepackage{ifthen}
\usepackage{tikz}
\usepackage{algorithm}
\usepackage{algorithmic}
\usepackage{pifont}
\usepackage{url}
\usepackage{mathtools}
\usepackage{algorithm,algorithmic}
\usepackage{hyperref}
\usepackage{mathrsfs}
\usepackage{multirow}
\RequirePackage{etex}

\renewcommand{\le}{\leqslant}
\renewcommand{\ge}{\geqslant}
\renewcommand{\hat}{\widehat}

\newtheorem{theorem}{Theorem}[section]
\newtheorem{corollary}{Corollary}[theorem]
\newtheorem{lemma}[theorem]{Lemma}
\newtheorem{assumption}[theorem]{Assumption}
\newtheorem{definition}[theorem]{Definition}

\newtheorem{proposition}[theorem]{Proposition}

\usepackage{pgfplotstable}
\usepackage{pgfplots}
\pgfplotsset{
	table/search path={plot_figures},
}
\usepackage{tikz}
\usepackage{xr}
\usetikzlibrary{external}
\usepackage{cite}
\usetikzlibrary{arrows,%
	petri,%
	topaths}%
\usetikzlibrary{graphs,graphs.standard}
\usepackage{tkz-berge}
\graphicspath{{Figures/}}
\usepackage{subfigure}
\pgfdeclarelayer{bg}    
\pgfsetlayers{bg,main}  
\pgfplotsset{compat=1.14}

\begin{document}

\title{Non-Bayesian Social Learning \\ with Uncertain Models }

\author{ James~Z.~Hare\textsuperscript{*},~C\'{e}sar A.~Uribe\textsuperscript{*},~Lance~Kaplan~\IEEEmembership{Fellow,~IEEE,}~Ali~Jadbabaie~\IEEEmembership{Fellow,~IEEE}
	\thanks{J. Hare and L. Kaplan are with the Signal and Image processing branch of the US Army Research Laboratory, Adelphi, MD 20783 USA (e-mail: james.z.hare31@gmail.com, lance.m.kaplan@us.army.mil). C.A. Uribe and A. Jadbabaie are with the Laboratory for Information and Decision Systems, and the Institute for Data, Systems, and Society, Massachusetts Institute of Technology, Cambridge, MA 02139 USA (email: cauribe@mit.edu, jadbabai@mit.edu). This research was sponsored by a Vannevar Bush Fellowship and OSD LUCI programs. }
	\thanks{\textsuperscript{*}JZH and CAU contributed equally.}
	}

\maketitle

\begin{abstract}
 Non-Bayesian social learning theory provides a framework that models distributed inference for a group of agents interacting over a social network. In this framework, each agent iteratively forms and communicates beliefs about an unknown state of the world with their neighbors using a learning rule. Existing approaches assume agents have access to precise statistical models (in the form of likelihoods) for the state of the world. However in many situations, such models must be learned from finite data. We propose a social learning rule that takes into account uncertainty in the statistical models using second-order probabilities. Therefore, beliefs derived from uncertain models are sensitive to the amount of past evidence collected for each hypothesis. We characterize how well the hypotheses can be tested on a social network, as consistent or not with the state of the world. We explicitly show the dependency of the generated beliefs with respect to the amount of prior evidence. Moreover, as the amount of prior evidence goes to infinity, learning occurs and is consistent with traditional social learning theory.

\end{abstract}


\IEEEpeerreviewmaketitle

\section{Introduction}
The theory of \emph{Non-Bayesian Social Learning} \cite{JMST2012} has gained increasing attention over the past few years as a scalable approach that models distributed inference of a group of agents interacting over a social network. Individually, each agent in the network may not be able to infer the true state of the world. Also, agents may only observe a small fraction of the total information, leading to conflicting beliefs. Additionally, the agent's measurement process or sensing modalities may lead to ambiguous decisions, to further hinder the inference problem. Thus, non-Bayesian social learning theory provides a framework that allows for heterogeneous data aggregation, enabling every agent in the network to form a consensus belief on the true state of the world. 

In this framework, each agent repeatedly forms and communicates their beliefs about an unknown state of the world with their neighbors using a social learning rule and the likelihood of a new observation conditioned on predefined statistical models. 
The social learning rule assumes \textit{bounded rationality}, i.e., the beliefs of the agent's neighbors are sufficient statistics, also known as \textit{imperfect recall}~\cite{MTJ18}, which considerably simplifies computing the joint beliefs. Calculating the joint beliefs does not require knowledge of the network structure, inter-dependencies, or historical beliefs of every agent in the network as in \textit{Bayesian social learning} theory~\cite{GK2003, ADLO2011, KT2013, RJM2017}. Furthermore, imperfect recall has been shown to guarantee the agents' beliefs  converge to the global Bayesian result almost surely~\cite{JMST2012}. 

One of the major assumptions in the current literature, is that the statistical models of each hypothesis are known precisely. This assumption requires that the agents collect a sufficiently large set of labeled training data to accurately model the parameters of the statistical models. However, in some situations, (e.g. data is too expensive/impossible to collect or the measurement process is imprecise) the agents may only receive labeled data for a subset of states, or an insufficient amount of training data, which leads to \textit{uncertain} model parameters.

In this work, we present a new non-Bayesian social learning method that takes into account uncertainties in the statistical models (i.e., hypotheses or likelihood functions). Classically, inferences are made by normalizing the statistical models over the set of hypotheses. In the uncertain case, the amount of prior evidence for each hypothesis may vary, causing the uncertain models to change significantly, making them incommensurable. We propose a generalized model that reflects the amount of prior evidence collected. We build up our results from the concept of uncertain likelihood ratios for decision making under uncertainty~\cite{J2016,R1984}. This allows us to evaluate the consistency of the prior evidence with the observation sequence to judge each hypothesis on its own merit. We study the convergence properties of the proposed method for two standard social aggregating rules, \emph{log-linear}~\cite{MTJ18} and \emph{DeGroot}~\cite{JMST2012}. We show that when the agents have a finite amount of prior evidence, the agents' beliefs asymptotically converge to a finite value between zero and infinity, which represents the consistency of the hypothesis with respect to the ground truth. Furthermore, we show that we can exactly quantify the point of convergence for update rules based on log-linear aggregation. Finally, we show that as the amount of prior evidence grows unboundedly, the beliefs of every hypothesis inconsistent with the ground truth converge to zero. This indicates that learning is possible with uncertain models and is consistent with classical non-Bayesian social learning theory. 

The remainder of this paper is organized as follows. First, in Section \ref{sec:LR} we present a review of the current literature in non-Bayesian social learning theory and uncertainty modeling. Then, we describe the problem and main results in Section \ref{sec:PF}. Next, we derive the uncertain statistical models in Section \ref{sec:NULF}. In Section \ref{sec:DNBLWUL}, we implement the uncertain models into the log-linear update rule and formally prove the main result. Then in Section \ref{sec:DG_Update}, we study the properties of the DeGroot-style update rule with the uncertain likelihood ratio. Finally, we provide a numerical analysis in Section \ref{sec:SIM} to empirically validate our results and conclude the paper in Section \ref{sec:Conclusion}.

\textbf{Notation:} Bold symbols represent a vector/matrix, while a non-bold symbol represents its element. We use the indexes $i$ and $j$ to represent agents, $t$ to constitute the time step, and $k$ to index the category. We use $[\mathbf{A}]_{ij}$ to represent the entry of matrix $\mathbf{A}$'s $i$th row and $j$th column. We denote $X\overset{P}{\to}Y$ to represent that the sequence $X$ converges in probability to $Y$. Furthermore, we abbreviate the terminology almost surely by a.s. and independent identically distributed as i.i.d..

\section{Literature Review} \label{sec:LR}

\subsection{Non-Bayesian Social Learning}

Much of the learning algorithms developed in the literature have been derived using distributed optimization strategies for a group of agents, which typically utilize gradient-decent methods \cite{YYZS2018}. These approaches construct their decentralized algorithm using a consensus strategy \cite{NO2009, KM2011, KMR2012} or a diffusion strategy \cite{LS2007, CS2012, CS2015p1, CS2015p2} to ensure that the agents learn the true state. At the same time, non-Bayesian social learning methods \cite{JMST2012} were developed to perform distributed inference of a true state using a DeGroot-style \cite{D1974} (arithmetic average) learning rule, where it has been shown in \cite{SJ2013} that the Bayesian learning approach is linked to the distributed optimization framework. Since then, this learning rule has been studied in strongly-connected and weakly-connected graphs which characterized the beliefs rate of convergence \cite{MJRT2013} and the effects of influential agents on the resulting beliefs \cite{SYS2017}, respectively. Furthermore, this rule has been identified as a boundary condition that ensures learning \cite{MTJ18}. 

The DeGroot-style learning rule was then extended by a stream of papers that studied a geometric average learning rule known as the log-linear rule \cite{RT2010, RMJ2014, SRJ2015, NOU2015, LJS2018}. These works found that the agents will converge to the ``Bayesian Peer Influence'' heuristic \cite{LR2018} in finite time for fixed graphs \cite{MTJ18, SRJ2015}, time-varying undirected graphs \cite{NOU2017}, and time-varying directed graphs \cite{NOU2015, NOU2016}. Much of the focus has been on developing learning rules that improve the convergence rate of the beliefs \cite{NOU2015, NOU2016}. This has lead to the development of the log-linear learning rule with one-step memory \cite{RMJ2014, NOU2017}, observation reuse \cite{BT2018}, and the accumulation of all observations \cite{SV2018}. However, the common assumption in the literature is that the likelihood functions are known precisely. Thus, this paper studies the Log-linear and DeGroot-style learning rules with uncertain models. 

\subsection{Uncertainty Models}
Modeling the uncertainty in statistical models has been approached from many different philosophies, including possibility theory \cite{DP2001,K2005,DP2012}, probability intervals \cite{W1996, W1997, B2005}, and belief theory \cite{S1976, SK1994}. These approaches extend traditional probability calculus to encompass uncertainty into the model parameters. This was then extended to the theory of subjective logic~\cite{J2016}, which constructs a subjective belief of the model that can be mapped into a second-order probability distribution. 

Second-order probability distributions \cite{GS1982, C1996} are typically modeled as the conjugate prior of the first-order distribution, which does not complicate the overall analysis and allows for a reduction in uncertainty as more information becomes available. In particular, an example of a second-order distribution is the Dirichlet distribution who's shape is governed by the amount of prior evidence collected. This has led to the development of the imprecise Dirichlet process \cite{W1996, DP2001, B2005}, which allows the likelihood parameters to be modeled within upper and lower probability bounds. 

From a Bayesian point of view, this approach was also studied by constructing the likelihood based on the posterior predictive distribution \cite{R1984, M1994}. This lead to many approaches on how to correctly construct the prior distribution to provide non-informative information and allow the posterior distributions to be data-dominated \cite{TGM2009} (see \cite{KW1996} for a detailed review). However, these studies did not consider the problem of developing a prior based on the amount of prior information available. In this work we utilize the Bayesian point of view which computes the likelihood based on the posterior predictive distribution, while borrowing concepts from subjective logic to model the prior Dirichlet distribution. 

\section{Problem Formulation, Algorithms and Results} \label{sec:PF}

\subsection{Signals, Hypotheses, and Uncertain Models} \label{sec:pf_ahu}

Consider a network of $m$ agents interacting over a social network, who are trying to collectively infer and agree on the \textit{unknown} state of the world $\theta^* \in \boldsymbol{\Theta}$, where $\boldsymbol{\Theta}=\{\theta_1,...,\theta_S\}$ is the set of possible states of the world. The agents gain information about the state $\theta^*$ via a sequence of realizations of an i.i.d. random variable conditioned on the state of the world being~$\theta^*$. Thus, given such observations, the agents seek to identify a hypothesis (i.e., a distribution for the random variable generating the observations), that best explains the observations and therefore the state of the world.

Each agent $i$ seeks to infer the underlying state of the world $\theta^*$ by sequentially collecting independent private signals $\{\omega_{it}\}_{t \geq { 1}}$, with $\omega_{it} \in \boldsymbol{\Omega} = \{1,\dots,K\}$ and $K\ge 2$ possible mutually exclusive outcomes, where the probability of observing an outcome $k\in\boldsymbol{\Omega}$ is $\pi_{k{i \theta^*}}$. Moreover, an agent keeps track of these realizations via a histogram $\mathbf{n}_{it}=\{n_{i1t},...,n_{iKt}\}$, s.t. $\sum_{k=1}^K n_{ikt} = t$ and $n_{ikt}$ is the number of occurrences of category $k$ up to time $t$. 

The vector $\mathbf{n}_{it}$ is a realization of $t$ draws from a multinomial distribution with parameters $\boldsymbol{\pi}_{i\theta^*}$. We call this distribution $P_{i\theta^*}$. However, our main assumption is that agents do not have a precise statistical model for the possible states of the world, i.e., the values of $ \{\boldsymbol{\pi}_{i\theta}\}_{\forall \theta \in \Theta}$ are partially unknown by the agents. Only limited information is available for each possible state of the world and decisions are made over \textit{uncertain likelihood models}. We will assume that agents construct these uncertain likelihood models from available prior partial information acquired via private signals for each possible state of the world. For a hypothesis $\theta$, an agent $i$ has available $R_{i\theta}$ independent trials. This provides the agent with a set of counts $\mathbf{r}_{i\theta}=\{r_{i1\theta},...,r_{iK\theta}\}$, denoted as the \emph{prior evidence} of hypothesis $\theta$, where $r_{ik\theta}\in [0, \infty)$ is the number of occurrences of outcome $k\in\boldsymbol{\Omega}$ and $\sum_{k=1}^K r_{ik\theta } = R_{i\theta}$. Thus, the vector of counts $\mathbf{r}_{i\theta}$ is a realization of a multinomial random variable with parameters $R_{i\theta}$ and $\boldsymbol{\pi}_{i\theta}$ for $i \in \{1,\dots,m\}$ and $\theta \in \boldsymbol{\Theta}$. Furthermore, when $R_{i\theta}$ is finite (not sufficiently large), the vector $\boldsymbol{\pi}_{i\theta}$ is \emph{uncertain}, and an agent cannot compute the probability distribution precisely.

To clarify the model above consider that an agent $i$ is handed a set of $K$ sided dice labeled $1,...,S$. Each die $s$ represents a hypothesis $\theta_s\in\boldsymbol{\Theta}$ and the parameters $\boldsymbol{\pi}_{i\theta_s}$ represents the set of probabilities of the die landing on each face. The agent only has access to each die for a small amount of time, where they roll the die and collect the counts of each face during each roll to construct the sets $\mathbf{r}_{i\theta}$ $\forall \theta \in \boldsymbol{\Theta}$. Then, all of the dice are collected and a new unlabeled die is presented to the agent. The goal of the agent is to identify which of the $S$ hypotheses best matches the distribution observed by rolling the new die. This is the main object of study of this paper: the design of a distributed algorithm that allows a group of agents to construct consistent beliefs about a set of hypotheses based on uncertain likelihood models.

\subsection{Social Learning with Uncertain Models}

Given the prior evidence for the set of hypotheses, the sequence of private observations and the interactions with the other agents in the network, an agent iteratively constructs beliefs over the hypotheses set $\boldsymbol{\Theta}$. We will denote the belief of an agent $i$ about a hypothesis $\theta$ at a time $t$ as $\mu_{it}(\theta)$. Moreover, the belief of agent $i$ about hypothesis $\theta$ at time $t+1$ will be a function of the tuple $\{\mathbf{r}_{i\theta},\mathbf{n}_{it}, \{\mu_{jt}(\theta)\}_{j \in \mathbf{M}^i} \}$, where $\mathbf{M}^i$ is the set of agents (or neighbors) that can send information to agent $i$.

We propose the following belief update rule, based on uncertain likelihood models,
\begin{align}\label{eq:main_algo}
\mu_{it+1}(\theta) = \ell_{i\theta}(\mathbf{n}_{it}, \omega_{it+1}|\mathbf{r}_{i\theta})\prod_{j\in \mathbf{M}^i}\mu_{jt}(\theta) ^{[\mathbf{A}]_{ij}},
\end{align}
where
\begin{align}\label{eq:el;_def}
\ell_{i\theta}(\mathbf{n}_{it},k|\mathbf{r}_{i\theta}) &= \frac{(r_{ik\theta} + n_{ikt}+1)(t+K-1)}{(R_{i\theta}+t+K-1)(n_{ikt} +1)},
\end{align}
$\mu_{i0}(\theta)=1$ $\forall i\in\{1,...,m\}$, and $[\mathbf{A}]_{ij}$ is the weight agent $i$ assigns to the belief shared by agent $j$.

Equation~(\ref{eq:main_algo}) is an aggregation step (a weighted geometric mean), and a normalized uncertain likelihood non-Bayesian update, where Equation (\ref{eq:el;_def}) is the uncertain likelihood ratio update based on the observed signal at time $t$. This proposed belief update rule will be motivated in Section \ref{sec:DNBLWUL}.

Note that the generated beliefs are not probability distributions since they are not normalized over the set of hypotheses $\boldsymbol{\Theta}$ as in traditional non-Bayesian social learning. Rather, the generated beliefs with uncertain likelihoods represents the consistency of the hypothesis with the ground truth given the accumulated prior evidence. A detailed description of the proposed inference rule will be presented in Section \ref{sec:NULF}.

\subsection{Assumptions and Definitions}

The agents social interactions are modeled as an exchange of beliefs over a weighted undirected graph $\mathcal{G}=(\mathbf{M},\mathbf{E})$, which consists of the set of agents $\mathbf{M}=\{1,...,m\}$ and a set of edges $\mathbf{E}$. An edge is defined as the connection between agent $i$ and $j$ and is denoted by the ordered pair $(i,j)\in \mathbf{E}$. The weights along each edge form an adjacency matrix, $\mathbf{A}$, which represents the amount of influence that agent $i$ has on agent $j$ (and vise versa) such that $[\mathbf{A}]_{ij}>0$ if $(i,j)\in \mathbf{E}$ and $[\mathbf{A}]_{ij}=0$ if $(i,j)\notin \mathbf{E}$. Furthermore, the set of neighbors of agent $i$ is defined as $\mathbf{M}^{i}=\{j\in \mathbf{M}|(i,j)\in \mathbf{E}\}$ and we assume that the agents within $\mathbf{M}^i$ report their beliefs truthfully.

\begin{assumption}\label{assum:graph}
	The graph $\mathcal{G}$ is undirected and connected. Moreover, the corresponding adjacency matrix $\mathbf{A}$ is doubly stochastic and aperiodic. Note that $\mathbf{A}$ is irreducible due to connectivity. 
\end{assumption}
	
Assumption \ref{assum:graph} is common among the consensus literature and allows the agents interactions to be represented by a Markov Chain. This guarantees convergence of the graph to a fully connected network and defines bounds on the second largest eigenvalue based on the number of agents \cite{NOU2017}. Note that it is not always possible to derive a directed graph with a doubly stochastic adjacency matrix (as provided in \cite{GC2010}) especially in a distributed manner. However, if the graph is undirected, then the distributed agents can construct a Lazy Metropolis matrix to form a doubly stochastic matrix. Furthermore, time-varying directed graphs can form doubly stochastic matrices using the push-sum algorithm \cite{NO2016}. 

\begin{assumption} \label{assum:uncertain_distributioin}
    Each agent $i$ at time $t=0$ has their counter for the observations of their private signals set to $n_{ik0}=0$ for all $i\in \mathbf{M}$ and $k\in\boldsymbol{\Omega}$. This enables the definition of the prior uncertain probability distribution $\widetilde{\mathcal{P}}_i(\mathbf{0}) = \{\widetilde{P}_{i\theta}(\mathbf{n}_{i0}=\mathbf{0}|\mathbf{r}_{i\theta})\}_{\forall \theta \in \boldsymbol{\Theta}}$ at time $t=0$, which are derived from the marginal of a second-order distribution of the probabilities $\boldsymbol{\pi}_{i\theta}$ given the prior evidence $\mathbf{r}_{i\theta}$ (to be derived in Section~\ref{sec:NULF}).
\end{assumption}

\begin{definition} \label{assum:dog_distribution}
    When agent $i$ collects an infinite (or a sufficiently large) amount of prior evidence for hypothesis $\theta$, the probabilities $\boldsymbol{\pi}_{i\theta}$ are known precisely and we say that the agent has a epistemically \emph{certain} statistical model for the hypothesis $\theta$, i.e., $\widetilde{\mathcal{P}}_i(\mathbf{0}) = \{P_{i\theta}(\mathbf{n}_{i0}=\mathbf{0}|\boldsymbol{\pi}_{i\theta})\}_{\forall \theta \in \boldsymbol{\Theta}}$.
\end{definition}

The precise definitions of the uncertain and certain likelihood models for a multinomial distribution will be formally introduced in Section \ref{sec:NULF}. Note that the usage of \emph{certain} statistical models is the same as \emph{dogmatic} opinions in subjective logic \cite{J2016}. 

We assume that the agents have calibrated their measurement models to allow them to distinctly identify the categories observed. However, it may be too expensive for the agents to conduct a sufficient number of trials to identify the probabilities $\boldsymbol{\pi}_{i\theta}$ of each hypothesis $\theta$ precisely.

Additionally, we allow the amount of prior evidence collected for each hypothesis can vary between hypotheses and agents, i.e. $R_{i\theta}\ne R_{i\hat{\theta}}$ for any $\hat{\theta}\ne \theta$ and $R_{i\theta}\ne R_{j\theta}$ for any $i\ne j$. This means that the distributions within $\widetilde{\mathcal{P}}_{i}$ are incommensurable, causing the traditional approach of normalizing $\widetilde{P}_{i\theta}$ over the set of $\widetilde{\mathcal{P}}_i$ to produce errors as an unintended consequence. Thus, we propose to normalize each distribution by a common \emph{vacuous} probability model that statistically models the agents ignorance of hypothesis $\theta$, i.e., $\widetilde{P}_{i\theta}(\mathbf{0}|\mathbf{r}_{i\theta}=\mathbf{0})$. A thorough discussion of this concept is presented in Section \ref{sec:NULF}.

Furthermore, we assume that the agent may face an identification problem due to (i) a varying amount of prior evidence and (ii) non-informative observations. The first condition is an effect of the proposed uncertain models, while the second condition is caused when multiple hypotheses $\hat{\boldsymbol{\Theta}}$ have the same probability distribution as the ground truth state of the world, s.t. $\hat{\boldsymbol{\Theta}}=\{\theta\in\boldsymbol{\Theta}|\boldsymbol{\pi}_{i\theta} = \boldsymbol{\pi}_{i\theta^*}\}$. However, for every hypothesis $\hat{\theta}\in \hat{\boldsymbol{\Theta}}$, we assume that there exists another agent $j$ that has informative observations for $\hat{\theta}$, s.t. $\boldsymbol{\pi}_{j\theta} \ne \boldsymbol{\pi}_{j\theta^*}$. Thus, the agents must collaborate to unequivocally identify the true state of the world.

Finally, we make the following assumption on the agents initial beliefs for each hypothesis.
\begin{assumption} \label{assum:inital_beliefs}
    The agents initial beliefs $\mu_{i0}(\theta)=1$ $\forall i\in\{1,...,m\}$ and $\forall \theta \in \boldsymbol{\Theta}$.
\end{assumption}
Assumption \ref{assum:inital_beliefs} allows the agents to express vacuous initial beliefs for each hypothesis based on the model of complete ignorance achieved by normalizing the uncertain probability distribution by the vacuous condition. This is also required to ensure that the beliefs evolve with time.

Next, we provide a definition of the posterior probability distribution of hypothesis $\theta$ for a centralized network.
\begin{definition}
    The \emph{centralized uncertain likelihood}  
    is the determination of the probability of the observations from all agents conditioned on the historical evidence for each hypothesis:
    \begin{equation}
    \widetilde{P}_\theta(\mathbf{n}_{1t},\mathbf{n}_{2t},...,\mathbf{n}_{mt}|\mathbf{r}_{1\theta},\mathbf{r}_{2\theta},...,\mathbf{r}_{m\theta}) = \prod_{i=1}^m \widetilde{P}_\theta(\mathbf{n}_{it}|\mathbf{r}_{i\theta}).
    \end{equation}
\end{definition}

Note that the decomposition of the centralized uncertain likelihood as the product of uncertain probabilities is only possible because the private signals as observations or evidence from training are statistically independent of each other and agents do not share their evidences $\mathbf{r}_{i\theta}$. As shown latter, the centralized uncertain likelihood and uncertain probabilities are sensitive to the amount of evidence, and it is more meaningful to normalize this value by the probability of the observations conditioned on no (or vacuous) historical evidence to form the centralized uncertain likelihood ratio:
    \begin{eqnarray}
    \prod_{i=1}^m \Lambda_{i\theta}(t) = \prod_{i=1}^m \frac{\widetilde{P}_{i\theta}(\mathbf{n}_{it}|\mathbf{r}_{i\theta})}{\widetilde{P}_{i\theta}(\mathbf{n}_{it}|\mathbf{0})}.
    \end{eqnarray}

The centralized uncertain likelihood ratio is achieved in a centralized network where a central node observes all of information. This distribution acts as the benchmark that the distributed agents should strive to achieve.

\subsection{Main Result}

We now present the main result of the paper. This result shows that the beliefs updated using the dynamics in Equation (\ref{eq:main_algo}) converge to a value with a one-to-one correspondence to the centralized uncertain likelihood ratio. The theorem is proven in Section \ref{sec:DNBLWUL}.

\begin{theorem} \label{thm:ULR_Con}
	Let Assumptions~\ref{assum:graph},~\ref{assum:uncertain_distributioin}, and ~\ref{assum:inital_beliefs} hold. Then, the beliefs generated by the update rule (\ref{eq:main_algo}) have the following property
	\begin{eqnarray} \label{eq:LL_Limit}
	 \lim_{t\to\infty}\mu_{it}(\theta) = \left( \prod_{j=1}^m \widetilde{\Lambda}_{j\theta} \right)^\frac{1}{m} , \ \text{a.s.}
	\end{eqnarray}
	where
	\begin{eqnarray} 
	\widetilde{\Lambda}_{j\theta} = \lim_{t \to \infty} \Lambda_{j\theta} (t) = \frac{B(\mathbf{1})}{B(\mathbf{r}_{j\theta}+\mathbf{1})}\prod_{k=1}^K (\pi_{jk\theta^*})^{r_{jk\theta}}, \ \text{a.s.}
	\end{eqnarray}
\end{theorem}

Theorem~\ref{thm:ULR_Con} states that the beliefs generated by the update rule \eqref{eq:main_algo} converges almost surely to the $m$th root of the product of the asymptotic uncertain likelihood ratio $\widetilde{\Lambda}_{i\theta}$ derived in Section \ref{sec:NULF}. Thus, with an abuse of notation, we will refer to the point of convergence of the beliefs $\mu_{it}(\theta)$ in the remainder of the paper as the centralized uncertain likelihood ratio. 

Note that the centralized uncertain likelihood ratio ranges between $[0,\infty)$ depending on the amount of prior evidence collected by the agents. When the agents have collected a finite amount of prior evidence, the probabilities $\boldsymbol{\pi}_{i\theta}$ $\forall i=1,..,m$ are uncertain, which results in the beliefs, $\mu_{it}(\theta)$ $\forall i=1,...,m$, converging to a finite value within $(0,\infty)$. Whereas, if the agents have collected an infinite amount of evidence, the probabilities $\boldsymbol{\pi}_{i\theta}$ are certain (known precisely) and the beliefs will converge to $0$ or diverge to $\infty$. This result will be presented in Section \ref{sec:DNBLWUL}.

The current literature identifies the hypothesis that minimizes the Kullback-Liebler (KL) divergence between the certain likelihood and the ground truth distribution. This allows only the beliefs of one hypothesis to converge to $1$, while the remaining beliefs converge to $0$, which allows for learning. Our result differs from the current literature, in that our uncertain beliefs converge to a finite value and multiple hypotheses may be accepted. However, when the agents are certain, only the hypothesis with a distribution that exactly matches the ground truth will be accepted, while any divergence between the distributions will cause the hypothesis to be rejected. This result follows the current literature under the closed world assumption that one of the predefined hypotheses is the ground truth. 

Next, we will present the derivation of the uncertain likelihood ratio and its properties, as well as define a test to evaluate the consistency of each hypothesis with the private signals.

\section{Uncertain Models Derivation} \label{sec:NULF}

In this section, we derive the uncertain models as a solution to incorporate the uncertainty about the statistical models for a set of hypotheses. For simplicity of exposition, throughout this section we will ignore the network, and assume the centralized scenario, i.e., there is only one agent. Thus, we will drop the $i$ in our notation. Later in Section \ref{sec:DNBLWUL} we will extend our results to the distributed network setup.

\subsection{Uncertain Likelihood Function via the Posterior Predictive Distribution}

\begin{figure*}[t] 
\centering 
	\subfigure[]{
		\includegraphics[width=.31\textwidth]{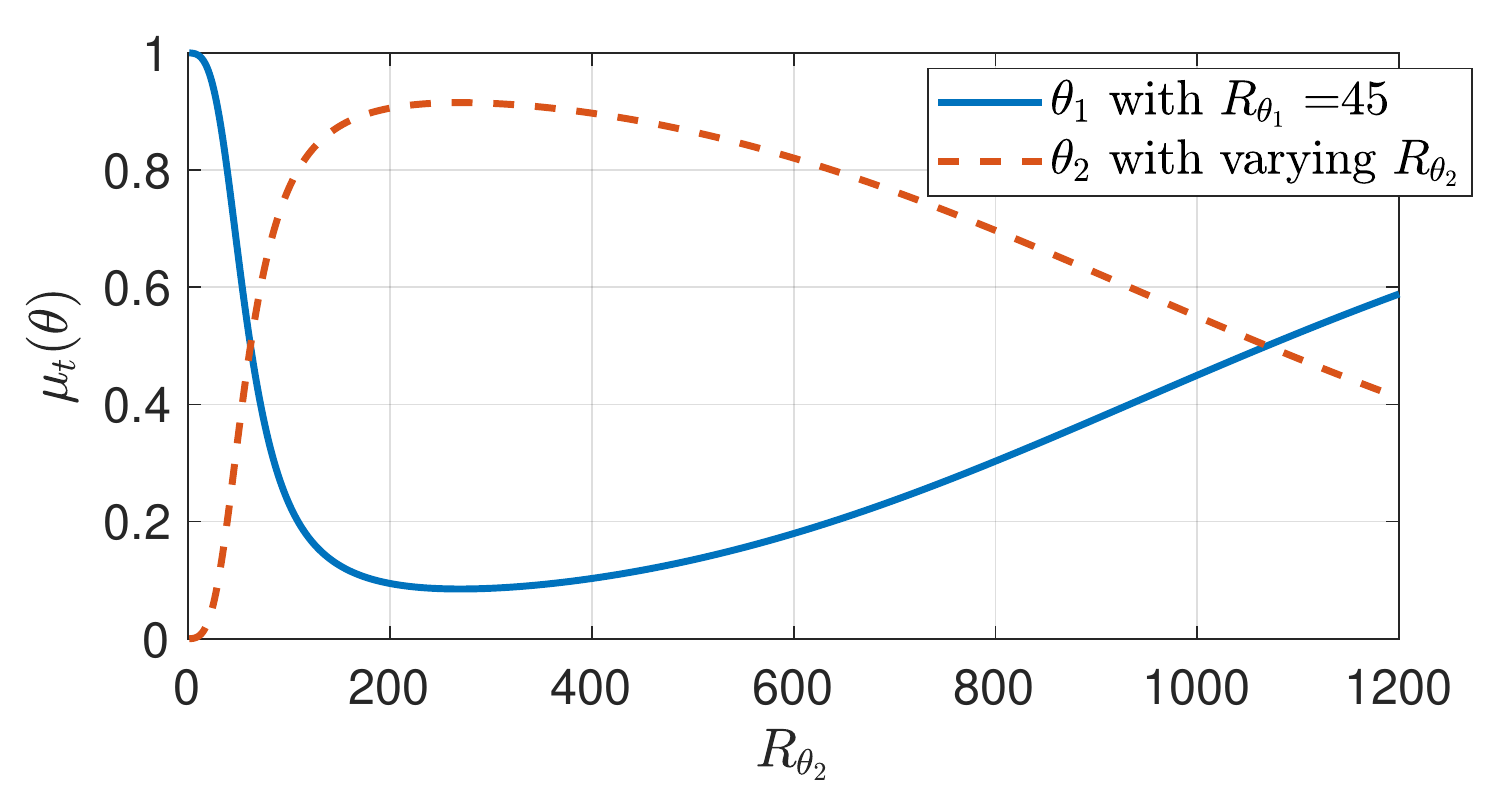}\label{fig:norm_Lambda1} 
	} \vspace{-3pt}
	\subfigure[]{
		\includegraphics[width=.31\textwidth]{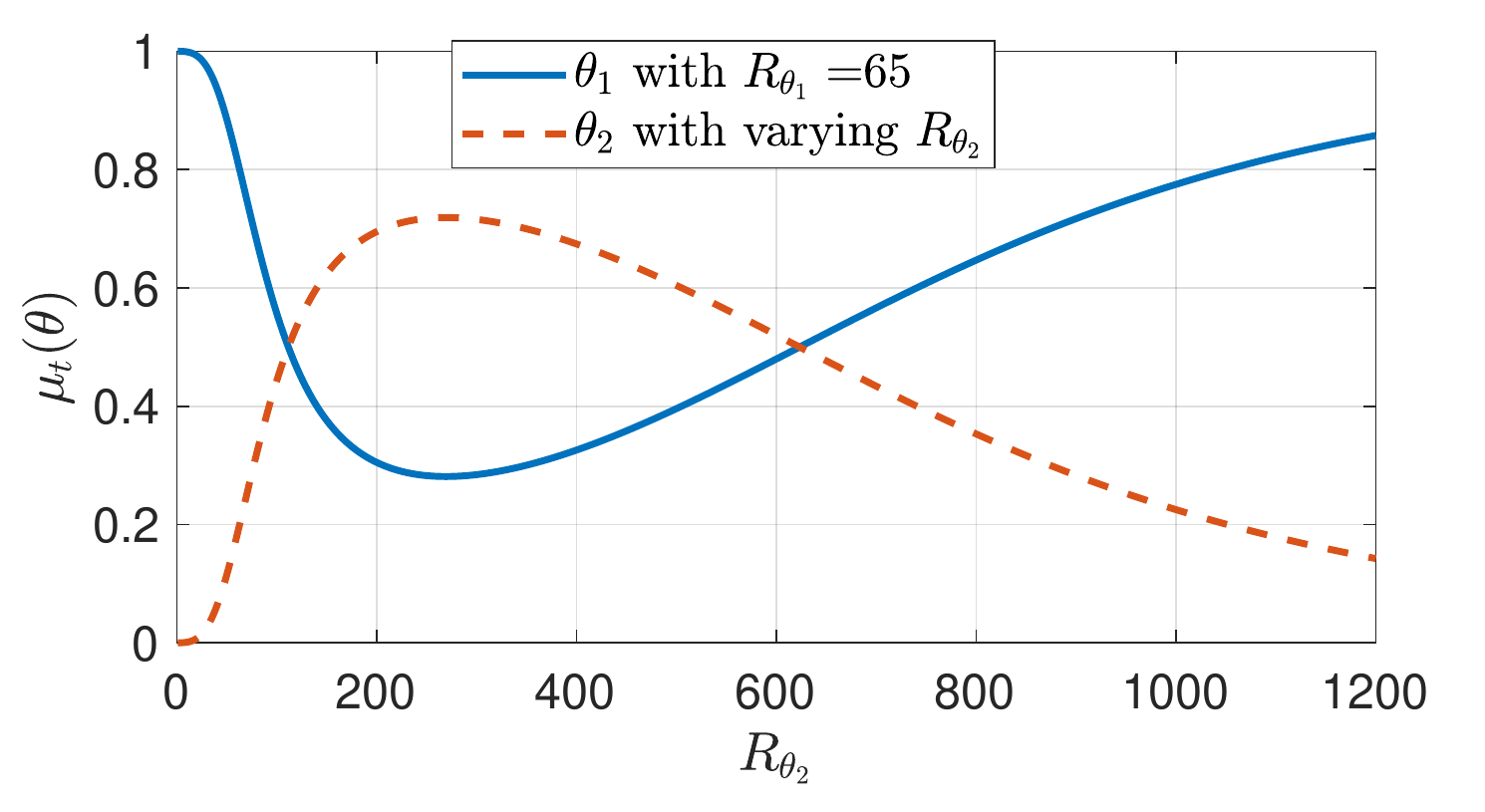}\label{fig:norm_Lambda2} 
	} \vspace{-3pt}
	\subfigure[]{
		\includegraphics[width=.31\textwidth]{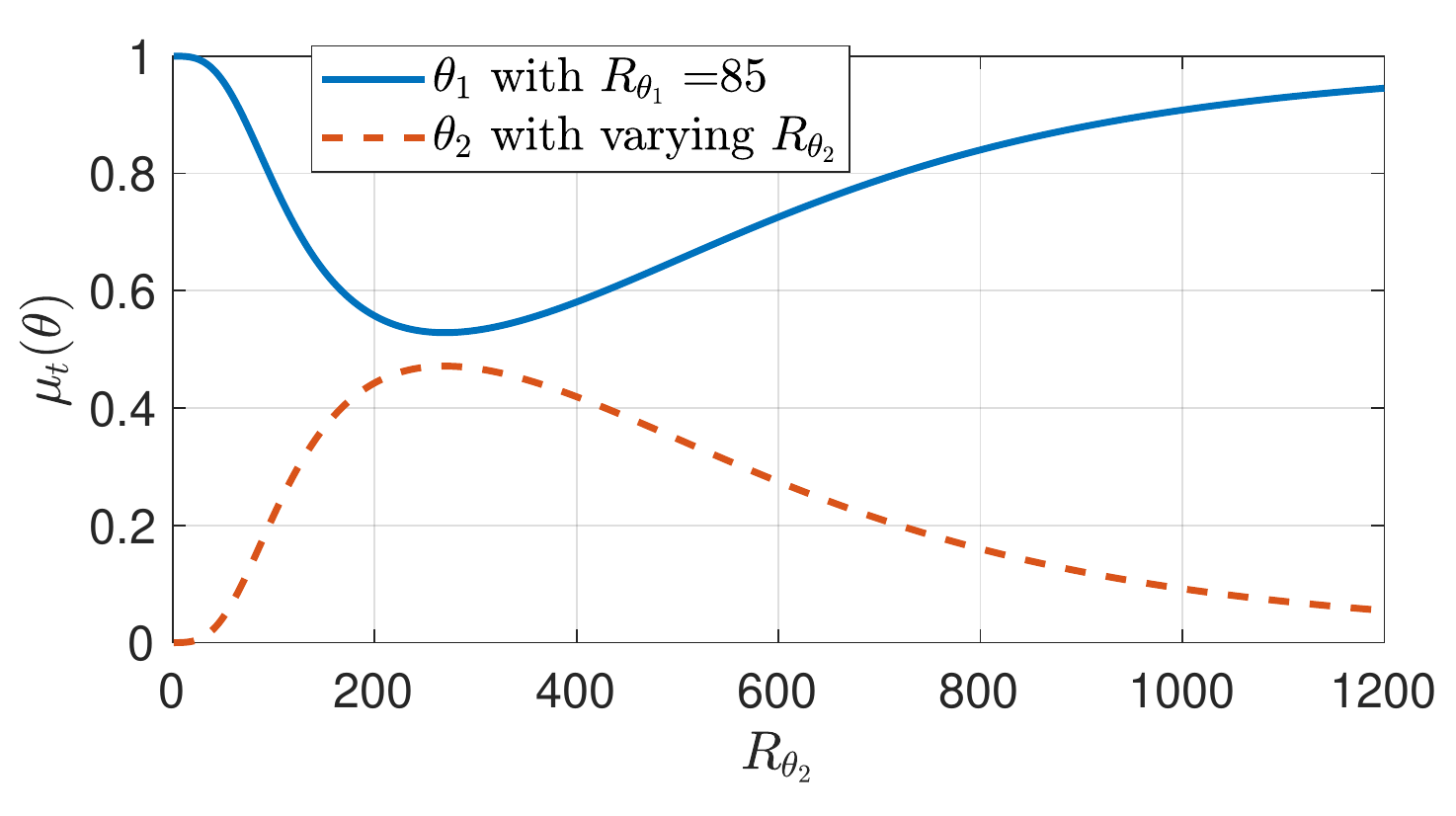}\label{fig:norm_Lambda3} 
} \vspace{-3pt}
	\caption{The normalized posterior predictive distribution (\ref{eq:exp_l}) using the update rule (\ref{eq:surrogate_likelihood}) versus the amount of evidence for hypothesis $\theta_2$ when $\mu_0(\theta_1)=\mu_0(\theta_2)=1$ and the evidence for hypothesis $\theta_1$ is: (a)~$R_{\theta_1} = 45$, (b)~$R_{\theta_1} = 65$, and (c)~$R_{\theta_1} = 85$.} \vspace{-15pt} \label{fig:norm_Lambda}
\end{figure*}

We model the uncertainty in the parameters of the multinomial distribution as a second-order probability density function. Similar approaches to modeling uncertainty have been presented in \cite{J2001, GS1982} and \cite{C1996}. As stated in Section~\ref{sec:pf_ahu}, an agent is assumed to construct its statistical model of hypothesis $\theta$ based on the prior evidence $\mathbf{r}_{\theta}$. Particularly, we are interested in a modified likelihood function that captures the uncertainty about the parameters $\boldsymbol{\pi}_\theta$ for each hypothesis based on finite samples.

Before the prior evidence $\mathbf{r}_{\theta}$ is presented, the agent  is assumed to have uniform prior belief about $\{\pi_{k\theta}\}_{k=1}^K$, thus $\{\pi_{k\theta}\}_{k=1}^K$ could be any point in the $K$-dimensional simplex,
\begin{equation*}
\mathcal{S}_K = \left \{\boldsymbol{\pi}  \bigg | \sum_{k=0}^K \pi_k = 1 \hspace{.05in} \mbox{ and $\pi_k \ge 0$ for $k=1,\ldots,K$} \right \},
\end{equation*}
with equal probability. However, once $\mathbf{r}_{\theta}$ is available, the agent updates its beliefs and constructs a posterior belief about $\{\pi_{k\theta}\}_{k=1}^K$. Particularly, if we assume the prior belief follows the uniform distribution over $\mathcal{S}_k$, and we observe $\mathbf{r}_\theta$ drawn from the multinomial distribution for hypothesis $\theta$, then the posterior belief is
\begin{eqnarray} \label{eq:dirichlet}
f (\boldsymbol{\pi}_{\theta}|\mathbf{r}_{\theta}) = \frac{\prod_{k=1}^K \pi_{k\theta }^{r_{k \theta }}}{B(\mathbf{r}_{\theta}+1)} \hspace{.05in}\mbox{s.t. $\boldsymbol{\pi} \in \mathcal{S}_K$,}
\end{eqnarray}
where $B(\alpha_1,...,\alpha_K)={\prod_{k=1}^K \Gamma(\alpha_k)}/{\Gamma(\sum_{k=1}^K \alpha_k)}$ is the $K$-dimensional Beta function~\cite{W1996}. The Dirichlet distribution is the conjugate prior of the multinomial distribution, which provides an algebraic convenience, and allows us to model the uncertainty of each parameter in the set $\boldsymbol{\pi}_{\theta}$ as a second-order probability density function. Clearly, as the number of observations in $\mathbf{r}_{\theta}$ increases, the posterior belief concentrates around $\boldsymbol{\pi}_{\theta}$.

In the social learning process an agent has collected $t$ signals $\boldsymbol{\omega}_{1:t}$ and has constructed its histogram $\mathbf{n}_{t}$. If the probabilities $\boldsymbol{\pi}_\theta$ are know absolutely, the agent would compute $P_\theta(\mathbf{n}_{t}|\boldsymbol{\pi}_\theta)$ as its likelihood function for the signal $\mathbf{n}_{t}$ given hypothesis $\theta$. However, in the uncertain condition, we must incorporate the finite knowledge about $\theta$ as $\widetilde{P}_\theta(\mathbf{n}_{t}|\mathbf{r}_\theta)$.

We propose the use of the posterior predictive distribution as the likelihood in lieu of the imprecisely known likelihood $P_\theta$. The posterior predictive distribution accounts for the uncertainty on $\boldsymbol{\pi}_{\theta}$, and it is calculated by marginalizing the distribution of $\mathbf{n}_{t}$ over the possible distributions of $\boldsymbol{\pi}_{\theta}$ given $\mathbf{r}_{\theta}$, i.e.,
	\begin{align} \label{eq:exp_l}
	\widetilde{P}_\theta(\mathbf{n}_{t}|\mathbf{r}_{\theta}) & =  \int_{\mathcal{S}_K} P_\theta(\mathbf{n}_{t}|\boldsymbol{\pi}_\theta) f(\boldsymbol{\pi}_{\theta}| \mathbf{r}_{\theta}) d\boldsymbol{\pi}_\theta, \nonumber \\
& =  \int_{\mathcal{S}_K} \prod_{k=1}^K \pi_{k\theta}^{n_{kt}}  f(\boldsymbol{\pi}_{\theta}| \mathbf{r}_{\theta}) d\boldsymbol{\pi}_\theta, \nonumber \\
	& =  \frac{B(\mathbf{r}_{\theta}+\mathbf{n}_{t}+\mathbf{1})}{B(\mathbf{r}_{\theta}+\mathbf{1})}.
	\end{align}
The uncertain likelihood function $\widetilde{P}_\theta$ represents the probability of the number of counts $\mathbf{n}_{t}$ of each category realized by the measurement sequence $\boldsymbol{\omega}_{1:t}$ conditioned on the prior evidence $\mathbf{r}_{\theta}$ for hypothesis $\theta$.

\vspace{-6pt}

\subsection{The Effects of Normalization with Uncertain Hypotheses}

Typically in Bayesian inference, a normalization step is used to ensure that the values are between $[0,1]$. Next, we will show that an update rule generated by using the posterior predictive distribution, as the uncertainty likelihood function, i.e.,
\begin{align}\label{eq:surrogate_likelihood}
\mu_t(\theta) & = \frac{\widetilde{P}_\theta(\mathbf{n}_{t}|\mathbf{r}_{\theta})\mu_{0}(\theta)}{\sum_{\nu \in \boldsymbol{\Theta}}\widetilde{P}_\nu(\mathbf{n}_{t}|\mathbf{r}_{\nu})\mu_{0}(\nu)},
\end{align}
is not robust to having dissimilar amounts of evidence for the different hypotheses. Thus, the following proposition holds.

\begin{proposition} \label{prop:normal}
Consider the update rule (\ref{eq:surrogate_likelihood}), with $\mu_0(\theta)>0$ $\forall \theta \in \boldsymbol{\Theta}\in\{\theta^*, \bar{\theta}\}$. Then, there exists a finite $R_{\theta^*}$ and $R_{\bar{\theta}}$ such that $Prob(\lim_{t\rightarrow \infty}\mu_t(\bar{\theta})>\lim_{t\rightarrow \infty}\mu_t(\theta^*))>0$.
\end{proposition}

Proposition \ref{prop:normal} states that due to the finite amount of evidence collected by the agent, the ground truth hypothesis will be rejected with a probability greater than 0. This occurs due to the following properties. First, if an insufficient amount of prior evidence is collected for hypothesis $\theta=\theta^*$, there is a probability greater than $0$ that the histograms $\mathbf{r}_\theta$ generated mismatch the ground truth parameters $\boldsymbol{\pi}_{\theta^*}$. Additionally, there is a probability greater than $0$ that the histograms generated for a hypothesis $\hat{\theta}\ne \theta^*$ could match the ground truth parameters. Thus, the hypothesis $\hat{\theta}$ would appear to be a better fit and be selected over the ground truth $\theta$. 

The second issue relates to the amount of prior evidence collected. Consider that the prior evidence for each hypothesis is consistent with their respective probability distribution, i.e., $\mathbf{r}_{\theta}=R_\theta \boldsymbol{\pi}_{\theta}$. However, consider that the amount of prior evidence collected for the ground truth hypothesis, say $\theta_1$, is smaller than some hypothesis $\theta_2$. Then, there is a chance that the belief update rule (\ref{eq:surrogate_likelihood}) of $\theta_2$ will be greater than $\theta_1$, as illustrated in Figure~\ref{fig:norm_Lambda}. 

As seen in Figure~\ref{fig:norm_Lambda1}, when $R_{\theta_1}=45$ and $R_{\theta_2}\in[100,1250]$, $\lim_{t \to \infty}\mu_t(\theta_2)>\lim_{t \to \infty}\mu_t(\theta_1)$, the ground truth will be rejected. However, as the amount of prior evidence increases to $R_{\theta_1}=65$ in Figure~\ref{fig:norm_Lambda2} and $R_{\theta_1}=85$ in Figure~\ref{fig:norm_Lambda3}, the range of $R_{\theta_2}$ that allows $\theta_1$ to be rejected decreases. Thus, there are scenarios that allow the probability of rejecting the ground truth to be greater than $0$ when using the update rule (\ref{eq:surrogate_likelihood}). Therefore, we cannot normalize over the set of hypotheses.   

We propose that the agents compare the posterior predictive distribution $\widetilde{P}_\theta$ to the model of complete ignorance, i.e., the \emph{vacuous} probability model. The vacuous probability model assumes that the agent has collected zero prior evidence for each hypothesis and strictly evaluates \eqref{eq:exp_l} with parameters $\mathbf{r}_{\theta}= \mathbf{0}$. This model considers that each probability $\pi_{ik\theta}$ is uniformly distributed in the simplex and represents complete uncertainty. Note that it follows from~\eqref{eq:exp_l} that
\begin{equation}
\widetilde{P}_\theta(\mathbf{n}_{t}|\mathbf{r}_{\theta} =\mathbf{0})= \frac{B(\mathbf{n}_t+1)}{B(\mathbf{1})}.
\end{equation}

Thus, we define the \emph{Uncertain Likelihood Ratio} as follows.

\begin{definition}[Uncertain Likelihood Ratio] \label{def:ULR}
	The uncertain likelihood ratio is the posterior predictive distribution normalized by the vacuous probability model, i.e. $R_{\theta}=0$, as follows:
		\begin{align} \label{eq:L}
		\Lambda_{\theta}(t) = \frac{\widetilde{P}_\theta(\mathbf{n}_{t}|\mathbf{r}_{\theta})}{\widetilde{P}_\theta(\mathbf{n}_{t}|\mathbf{0} 
)}= \frac{B(\mathbf{r}_{\theta} + \mathbf{n}_{t}+\mathbf{1})B(\mathbf{1})}{B(\mathbf{r}_{\theta}+\mathbf{1})B(\mathbf{n}_{t}+\mathbf{1})}.
		\end{align}
\end{definition}

Since the agent has different amounts of prior evidence for each hypothesis, the uncertain likelihood ratio cannot be evaluated over the set of all hypothesis as in \eqref{eq:surrogate_likelihood}. Thus, we propose that the agent evaluates each hypothesis individually utilizing the \emph{Uncertain Likelihood Ratio Test}.

\begin{definition}[Uncertain Likelihood Ratio Test] \label{def:ULRT}
	The uncertain likelihood ratio test is a likelihood ratio test that utilizes the uncertain likelihood ratio to evaluate the consistency of the prior evidence of hypothesis $\theta$ with the ground truth $\theta^*$. This test results in the following conclusions: 
	\begin{enumerate}
		\item If $\Lambda_{\theta}(t)$ converges to a value above one, there is evidence to accept that $\theta$ is consistent with the ground truth $\theta^*$.  Higher values indicate more evidence to accept $\theta$ as being equivalent to the ground truth.
		\item If $\Lambda_{\theta}(t)$ converges to a value below one, there is evidence to reject that $\theta$ is the ground truth $\theta^*$. Lower values indicate more evidence to reject $\theta$ as $\theta^*$. 
		\item  If $\Lambda_{\theta}(t)$ converges to a value near one, there is not enough evidence to accept or reject $\theta$ as $\theta^*$. 
	\end{enumerate}
\end{definition}
As a practical matter, one can define a threshold $\upsilon>1$ so that the hypothesis is deemed accepted, rejected or unsure if $\Lambda_\theta(t) \geq \upsilon$, $\Lambda_\theta(t) < 1/\upsilon$ and $1/\upsilon \leq \Lambda_\theta(t) < \upsilon$, respectively.\footnote{This choice of thresholds induces a set of symmetric thresholds $\pm \log(\upsilon)$ for $\log\left(\Lambda_\theta(t)\right) \in (-\infty,\infty)$.} The exact choice of thresholds is application dependent to balance the number of false positives and false negatives. Furthermore, the choice of threshold may be chosen based on the amount of prior evidence the agent has for hypothesis $\theta$. The construction of this threshold and its effects on the overall inference is out of the scope of this paper and thus left for future work.

The uncertain likelihood ratio test incorporates a third conclusion into the traditional likelihood ratio test which is a direct result of the agents uncertainty in the hypothesis. The current literature assumes a closed world and that the agent must select the hypothesis that best matches the observed data. However, when uncertainty is incorporated, the agents should judge each hypothesis on its own merits, i.e., how well it matches the observations relative to the historical evidence about that hypothesis.  For some hypotheses, there may not be enough evidence to accept or reject it.  Furthermore, there may be evidence to accept multiple hypotheses, but the wrong hypothesis exhibits a larger uncertain likelihood ratio as evident in Figure~\ref{fig:norm_Lambda}. Therefore, the inference problem is reformulated to accept the following set of hypotheses:
\begin{eqnarray}
\hat{\boldsymbol{\Theta}} = \{\theta\in\boldsymbol{\Theta}|\Lambda_{\theta}(t) \ge \upsilon 
\}.
\end{eqnarray}

\subsection{Asymptotic Behavior of the Centralized Uncertain Likelihood Ratio}

The inference drawn from the uncertain likelihood ratio test depends on the amount of prior evidence collected by the agent. This subsection studies the asymptotic properties of the uncertain likelihood ratio as $t\rightarrow \infty$. Particularly, we will assume a centralized scenario where there is only one agent, and we will observe the asymptotic behavior of its beliefs.

\begin{lemma} \label{lem:ULR_lim}
	The uncertain likelihood ratio in \eqref{eq:L} of hypothesis $\theta$ has the following property
	\begin{eqnarray} \label{eq:ULR}
	\widetilde{\Lambda}_{\theta} = \lim_{t \to \infty} \Lambda_{\theta}(t) = \frac{B(\mathbf{1})}{B(\mathbf{r}_{\theta}+\mathbf{1})}\prod_{k=1}^K \pi_{k\theta^*}^{r_{k\theta}}, \ \ \text{a.s.},
	\end{eqnarray}
	where $\mathbf{r}_{\theta}$ is the prior evidence about hypothesis $\theta$ and $\boldsymbol{\pi}_{\theta^*}$ are the ground truth probabilities. 
\end{lemma}

\begin{proof}
	First, the uncertain likelihood ratio can be expressed as
	\begin{eqnarray*}
	\Lambda_{\theta}(t) = \frac{B(\mathbf{1})\Gamma(t+K)\prod_{k=1}^K \Gamma(r_{k\theta }+{n_{kt}}+1)}{B(\mathbf{r}_{\theta}+1) \Gamma(R_{\theta} + t + K) \prod_{k=1}^K \Gamma({n_{kt}}+1)}.
	\end{eqnarray*}
	
	For a large $t$, we can approximate the ratio of gamma functions using Stirling's series \cite{Laforgia12}, where
	\begin{align*}
	\frac{\Gamma(x + \alpha)}{\Gamma(x + \beta)} = x^{\alpha -\beta} \left(1 +\frac{(\alpha -\beta)(\alpha - \beta-1)}{2x} + O(x^{-2})\right).
	\end{align*}
	
	Thus,
	\begin{align*}
	 \frac{\Gamma(r_{k\theta} + n_{kt} +1)}{\Gamma(n_{kt} +1)} & = n_{kt}^{r_{k\theta }} \left(1 +\frac{r_{k\theta }(r_{k\theta}-1)}{2n_{kt}} + O(n_{kt}^{-2})\right)
	\end{align*}
	$\forall k \in \boldsymbol{\Omega}$, and
	\begin{align*}
	\frac{\Gamma(t+K)}{\Gamma(t+K+R_\theta)} & =   t^{-R_\theta} \left(1 +\frac{-R_\theta(-R_\theta-1)}{2t} + O(t^{-2})\right) .
	\end{align*}
	
	Then, the limit of the uncertain likelihood ratio as $t\rightarrow \infty$ becomes
	\begin{multline*}
	\lim_{t\rightarrow \infty} \Lambda_{\theta}(t) =    \lim_{t\rightarrow \infty}   t^{-R_\theta} \left(1 +\frac{-R_\theta(-R_\theta-1)}{2t} + O(t^{-2})\right) \\ \cdot \prod_{k=1}^K n_{kt}^{r_{k\theta}} \left(1 +\frac{r_{k\theta}(r_{k\theta}-1)}{2n_{kt}} + O(n_{kt}^{-2})\right)\cdot \frac{B(\mathbf{1})}{B(\mathbf{r}_{\theta}+1)}.  \nonumber
	\end{multline*}
	
    Note that
	\begin{align*}
	\lim_{t\rightarrow \infty}\left(1 +\frac{r_{k\theta}(r_{k\theta}-1)}{2n_{kt}} + O(n_{kt}^{-2})\right) = 1,
	\end{align*}
	and
	\begin{align*}
	\lim_{t\rightarrow \infty} \left(1 +\frac{R_\theta(R_\theta+1)}{2t} + O(t^{-2})\right) =1.
	\end{align*}
	Then,
		\begin{align*}
	\lim_{t\rightarrow \infty} \Lambda_{\theta}(t) & =  \frac{B(\mathbf{1})}{B(\mathbf{r}_{\theta}+\mathbf{1})}\prod_{k=1}^K \pi_{k\theta^*}^{r_{k\theta}},
	\end{align*}
	with probability $1$ by the strong law of large numbers.
\end{proof}

\begin{figure}[t]
	\centering
	\includegraphics[width=0.95\columnwidth]{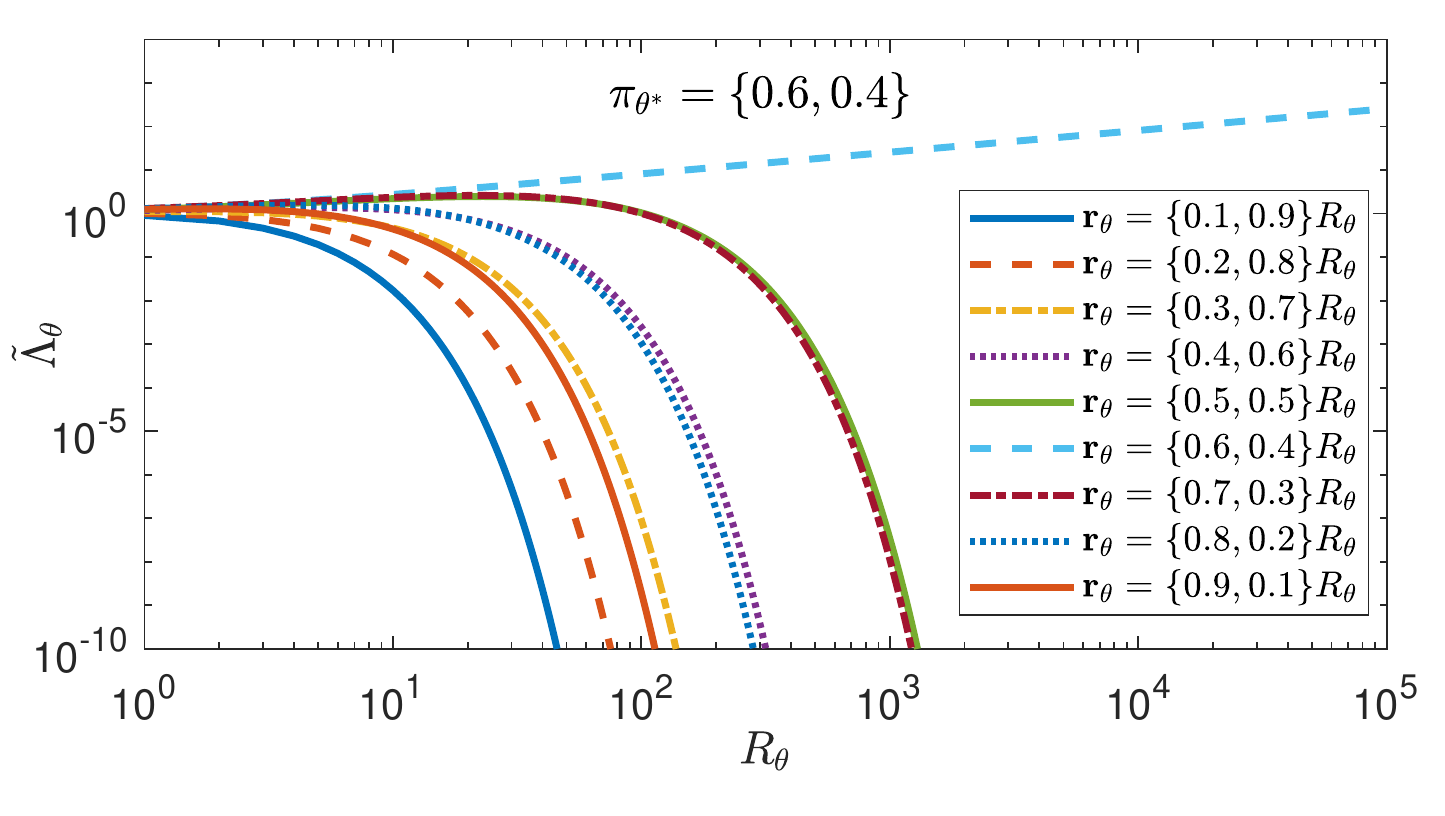}
	\caption{Asymptotic uncertain likelihood ratio vs. the amount of prior evidence $R_{\theta}$ for various hypothesis $\theta$ that differ from the ground truth $\theta^*$ with varying degrees of divergence. These curves assume that $\mathbf{r}_{\theta} = R_\theta \boldsymbol{\pi}_{\theta}$.} \vspace{-12pt}
	\label{fig:L_tilde}
\end{figure}

The effect of the prior evidence on $\widetilde{\Lambda}_\theta$ can be seen in Figure~\ref{fig:L_tilde}, where the asymptotic uncertain likelihood ratio vs. the amount of prior evidence $R_{\theta}$ is presented for various hypotheses. In this example, we consider $K=2$ and that the ground truth probabilities are $\boldsymbol{\pi}_{\theta^*}=\{0.6, 0.4\}$. Each curve in Figure \ref{fig:L_tilde} represents ideal conditions where the prior evidence is $\mathbf{r}_{\theta}=\{\pi, (1-\pi)\}R_{\theta}$, for $\pi\in\{0.1,0.2,...,0.9\}$. This result shows that for a finite amount of prior evidence, $\widetilde{\Lambda}_\theta$ converges to a finite value between $(0,\infty)$. Additionally, this shows that for small amounts of prior evidence, there are some hypotheses that produce an asymptotic uncertain likelihood ratio that is greater than $1$. Although, as the amount of prior evidence increases, the hypotheses with $\frac{\mathbf{r}_{\theta}}{R_{\theta}}\ne \boldsymbol{\pi}_{\theta^*}$ eventually decrease to $0$.

This result shows the effect of drawing conclusions using uncertain models. If the agent does not have enough prior evidence about a hypothesis, the asymptotic uncertain likelihood ratio will converge to a value around $1$, which falls into the third conclusion of the uncertain likelihood ratio test. As the amount of prior evidence increases, the asymptotic uncertain likelihood ratio for hypotheses with a small KL divergence, i.e., $D_{KL}(\frac{\mathbf{r}_{\theta}}{R_{\theta}}||\boldsymbol{\pi}_{\theta^*})\approx 0$, will converge to a value bigger than $1$, which results in the agent accepting the hypotheses. However, as the KL divergence and the amount of prior evidence increases, the asymptotic uncertain likelihood ratio converges to a value less than $1$, which is therefore rejected according to the uncertain likelihood ratio test.

Furthermore, Figure \ref{fig:L_tilde} provides an understanding of the asymptotic uncertain likelihood ratio as the amount of evidence increases to infinity, i.e., the agent becomes certain. This result is analytically characterized in the following Corollary.

\begin{corollary} \label{lem:dog_ULR}
    For an infinite amount of prior evidence, i.e., $R_{\theta}\rightarrow \infty$, the asymptotic uncertain likelihood ratio, $\widetilde{\Lambda}_{\theta}$, diverges to
	\begin{eqnarray}
	\lim_{t\rightarrow \infty} \widetilde{\Lambda}_{\theta}(t) = \infty & if \ \boldsymbol{\pi}_{\theta}=\boldsymbol{\pi}_{\theta^*} \ \text{a.s.}, \ \text{and} \end{eqnarray}
	converges to 
	\begin{eqnarray}
	\lim_{t\rightarrow \infty} \widetilde{\Lambda}_{\theta}(t) = 0 & if \ \boldsymbol{\pi}_{\theta}\ne\boldsymbol{\pi}_{\theta^*} \  \text{a.s.}
\label{eq:dog_ULR}
	\end{eqnarray}
\end{corollary}

\begin{proof}
First, by (\ref{eq:ULR}), the uncertain likelihood ratio converges to a function of a Dirichlet distribution evaluated at the ground truth probabilities $\boldsymbol{\pi}_{\theta^*}$, i.e.,
$\widetilde{\Lambda}_\theta = B(\mathbf{1}) f(\boldsymbol{\pi}_{\theta^*}|\mathbf{r}_\theta)$.
The expected value and variance of $f(\boldsymbol{\pi}_{\theta^*}|\mathbf{r}_\theta)$ are $E[\pi_{k\theta}]=\frac{r_{k\theta}+1}{R_\theta +K}$ and $Var[\pi_{k\theta}] = \frac{(r_{k\theta} + 1)(R_\theta - r_{k\theta} + K - 1)}{(R_\theta + K)^2(R_\theta +K+1)}$. Then, as $R_\theta \to \infty$, $E[\pi_{k\theta}]\to \pi_{k\theta}$ and $Var[\pi_{k\theta}]\to 0$ a.s. due to the strong law of large numbers. Therefore, $f(\boldsymbol{\pi}_{\theta^*}|\boldsymbol{r}_\theta)= \delta(\pi_{k\theta^*}-\pi_{k\theta})$ a.s., where $\delta(\cdot)$ is the Dirac delta function. This causes the asymptotic uncertain likelihood ratio to diverge to $\infty$ if $\boldsymbol{\pi}_{\theta^*}=\boldsymbol{\pi}_\theta$ and converge to $0$ if $\boldsymbol{\pi}_{\theta^*}\ne\boldsymbol{\pi}_\theta$.
\end{proof}

Corollary \ref{lem:dog_ULR} shows the relationship of the uncertain likelihood ratio with the assumption typically presented in non-Bayesian social learning literature. When the amount of prior evidence tends to infinity, the set of hypotheses with $\boldsymbol{\pi}_{\theta}=\boldsymbol{\pi}_{\theta^*}$ will be accepted, while the remaining hypotheses will be rejected since $\widetilde{\Lambda}_{\theta}=0$. This becomes the classical result, except that our definition of the uncertain likelihood ratio ranges from $[0,\infty)$ rather than $[0,1]$. Therefore, our uncertain model generalizes the certain likelihood assumption by forming an analytical expression of the likelihood as a function of the prior evidence.

Overall, one can view the amount of prior evidence $R_\theta$ as the amount of precision for knowledge about $\boldsymbol{\pi}_\theta$.  Larger $R_\theta$ provides the opportunity for a larger uncertain likelihood ratio as long as $\boldsymbol{\pi}_\theta = \boldsymbol{\pi}_{\theta^*}$. However, larger $R_\theta$ also means that the uncertain likelihood ratio is more likely to drop below one as the divergence between $\boldsymbol{\pi}_\theta$ and $\boldsymbol{\pi}_{\theta^*}$ increases.  As $R_\theta \to \infty$, any small divergence is enough for the uncertain likelihood ratio to go to zero. One could view the idea that traditional social learning actually selects the hypothesis that has the smallest KL divergence with the observations (e.g., see \cite{NOU2015}) as an admission that the underlying models $\boldsymbol{\pi}_\theta$ are not precise enough to match the ground truth precisely. The uncertain likelihood ratio developed in this section provides a formal method to evaluate the hypotheses based upon that lack of precision.

\section{Distributed Non-Bayesian Learning with Uncertain Models} \label{sec:DNBLWUL}

Thus far, we have derived the uncertain likelihood ratio for an agent $i$ that has received a set of measurements $\boldsymbol{\omega}_{i1:t}$ up to time $t\ge 1$. However, in non-Bayesian social learning theory, the agent's belief $\mu_{it}(\theta)$ is updated using the likelihood of the measurement $\omega_{it+1}$ given that hypothesis $\theta$ is the ground truth, not the uncertain likelihood ratio over the entire sequence of measurements. Therefore, in order to incorporate the uncertain likelihood ratio into non-Bayesian social learning, we must derive the uncertain likelihood ratio update function.
\begin{lemma} \label{lem:ULR_Up}
	Given that agent $i$ receives the measurement $\omega_{it}=k$ at time $t$, then the uncertain likelihood ratio update function $\ell_{i\theta}(\mathbf{n}_{it-1},k|\mathbf{r}_{i\theta})$ at time $t$ is defined as {
		\begin{eqnarray} \label{eq:ell_update}
		\ell_{i\theta}(\mathbf{n}_{it-1},k|\mathbf{r}_{i\theta}) &=& \frac{(r_{ik\theta} + n_{ikt-1}+1)}{(R_{i\theta}+t+K-1)} \frac{(K+t-1)}{(n_{ikt-1} +1)} \nonumber \\ &= & \frac{\hat{\pi}_{r_{ik\theta}}}{\hat{\pi}_{0}},
		\end{eqnarray}
		where $\hat{\pi}_{r_{ik\theta}}=\frac{r_{ik\theta} + n_{ikt-1}+1}{R_{i\theta}+t+K-1}$ and $\hat{\pi}_{0}=\frac{n_{ikt-1} +1}{K+t-1}$ are estimates of the private signal probabilities incorporating the prior evidence and not, respectively.
		This allows the uncertain likelihood ratio to be expressed in the following recursive form
		\begin{eqnarray} \label{eq:URL_Prod_Up}
		\Lambda_{i\theta}(t) =\ell_{i\theta}(\mathbf{n}_{it-1},k|\mathbf{r}_{i\theta})
\Lambda_{i\theta}(t-1). 
		\end{eqnarray} }
\end{lemma}

\begin{proof}
The uncertain likelihood ratio update is derived by expressing $\Lambda_{i\theta}(t)$ as a series of telescoping products.
	\begin{eqnarray} \label{eq:URL_tell}
	\Lambda_{i\theta}(t) =  \prod_{\tau=1}^t \ell^{i\theta}(\mathbf{n}_{i\tau-1},k|\mathbf{r}_{i\theta}) = \prod_{\tau=1}^t \frac{\Lambda_{i\theta}(\tau)}{\Lambda_{i\theta}(\tau-1)}, \nonumber 
	\end{eqnarray}
	since $\Lambda_{i\theta}(0)=1$. Therefore,
	\begin{eqnarray} \label{eq:ell}
	& & \ell_{i\theta}(\mathbf{n}_{i\tau-1},k|\mathbf{r}_{i\theta}) = \frac{\Lambda_{i\theta}(\tau)}{\Lambda_{i\theta}(\tau-1)} \nonumber \\ &=& \frac{B(\mathbf{r}_{i\theta}+\mathbf{n}_{i\tau}+\mathbf{1}) B(\mathbf{n}_{i\tau-1}+\mathbf{1})}{B(\mathbf{n}_{i\tau}+\mathbf{1})B(\mathbf{r}_{i\theta}+\mathbf{n}_{i\tau-1}+\mathbf{1})} \nonumber \\
	&=& \frac{\Gamma(R_{i\theta}+K+\sum_{k=1}^K n_{ik\tau-1})\Gamma(K+\sum_{k=1}^K n_{ik\tau})}{\Gamma(R_{i\theta}+K+\sum_{k=1}^K n_{ik\tau})\Gamma(K+\sum_{k=1}^K n_{ik\tau-1})} \nonumber \\
	& & \cdot \prod_{k=1}^K \frac{\Gamma(r_{ik\theta} + n_{ik\tau}+1)\Gamma(n_{ik\tau-1}+1)}{\Gamma(r_{ik\theta} + n_{ik\tau-1}+1)\Gamma(n_{ik\tau}+1)}
	\end{eqnarray}
	Then, if $\omega_{i\tau}=k$ is received, $n_{ik\tau}=n_{ik\tau-1}+1$ and $n_{i\bar{k}\tau}=n_{i\bar{k}\tau-1}$ for all $\bar{k} \in \boldsymbol{\Omega}\setminus \{k\}$. Recall that $\sum_{k=1}^K n_{ikt} = t$.
Therefore because $\Gamma(x+1)=x\Gamma(x)$, (\ref{eq:ell}) simplifies to
	\begin{eqnarray} \label{eq:ell1}
	\ell_{i\theta}(\mathbf{n}_{i\tau-1},k|\mathbf{r}_{i\theta}) &=& \frac{(r_{ik\theta} + n_{ik\tau-1}+1)(K+\tau-1)}{(R_{i\theta}+\tau+K-1)(n_{ik\tau-1} +1)}. \nonumber
	\end{eqnarray}
\end{proof}

The likelihood of the measurement $\omega_{it+1}$ given that hypothesis $\theta$ is the ground truth provides the following intuition. The numerator $\hat{\pi}_{r_{ik\theta}}$ represents the estimate of $\pi_{ik\theta}$ given the prior evidence $r_{ik\theta}$ and accumulated counts $n_{ik\theta}$, while the denominator $\hat{\pi}_0$ represents the estimate of $\pi_{ik\theta^*}$ given $0$ prior evidence and the accumulated counts $n_{ik\theta}$. The estimate $\hat{\pi}_0 \to \pi_{ik\theta^*}$ as $t\rightarrow \infty$ a.s. due to the strong law of large number, whereas the estimate $\hat{\pi}_{r_{ik\theta}}$ will converge based on the amount of prior evidence. If the prior evidence is finite, $\hat{\pi}_{r_{ik\theta}}\to \pi_{ik\theta^*}$ as $t\rightarrow \infty$ a.s., while as $R_{i\theta}\to\infty$, $\hat{\pi}_{r_{ik\theta}}\to \pi_{ik\theta}$ $\forall t\ge 0$ a.s. due to the strong law of large numbers. These properties are captured in the following lemmas.

\begin{lemma} \label{lem:update_lim}
    The likelihood of the measurement $\omega_{it+1}=k$ given that hypothesis $\theta^*$ is the ground truth has the following properties:
    \begin{enumerate}
        \item For finite evidence $R_{i\theta}$, $\lim_{t\rightarrow \infty} \ell_{i\theta}(\mathbf{n}_{it},k|\mathbf{r}_{i\theta}) = 1$,  $\forall k\in\boldsymbol{\Omega}$ a.s., and
        \item For infinite evidence (i.e., $R_{i\theta} \to \infty$), $\lim_{t\rightarrow \infty} \ell_{i\theta}(\mathbf{n}_{it},k|\mathbf{r}_{i\theta}) = \frac{\pi_{ik\theta}}{\pi_{ik\theta^*}}$, $\forall k\in\boldsymbol{\Omega}$ a.s..
    \end{enumerate}
\end{lemma}

\begin{proof}
First, since each private signal $\omega_{i\tau} \in \boldsymbol{\omega}_{i1:t}$ is i.i.d. and drawn from the $K$-state multinomial distribution with probabilities $\boldsymbol{\pi}_{i\theta^*}$, the strong law of large numbers leads to $\frac{n_{ikt}}{t} \to \pi_{ik\theta^*}$ for all $k\in \Omega$ a.s. Then, since $\hat{\pi}_0=\frac{\frac{n_{ikt}}{t}t+1}{t+K-1}$ and is continuous at $\pi_{ik\theta^*}$, $\hat{\pi}_0(\frac{n_{ikt}}{t})\to \pi_{ik\theta^*}$ with probability $1$ as $t\to \infty$. Similarly, when the prior evidence is finite and since $\hat{\pi}_{r_{ik\theta}}=\frac{r_{ik\theta}+\frac{n_{ikt}}{t}t+1}{R_{i\theta}+t+K-1}$ is continuous at $\pi_{ik\theta^*}$, then $\hat{\pi}_{r_{ik\theta}}(\frac{n_{ikt}}{t})\to \pi_{ik\theta^*}$ with probability $1$ as $t\to\infty$. Thus, when the prior evidence is finite and $\pi_{ik\theta^*}>0$, condition 1 holds. Furthermore, if $\pi_{ik\theta^*}=0$, the private signal $\omega_{it}=k$ will never be received when $\theta=\theta^*$. Thus
    $\hat{\pi}_{r_{ik\theta}} = \hat{\pi}_0$ as they both go to zero as $t \to \infty$ and
    condition 1 still holds.

    When the amount of prior evidence for hypothesis $\theta$ tends to infinity and is drawn from the distribution $\boldsymbol{\pi}_{i\theta}$, the strong law of large numbers leads to $\frac{r_{ik\theta}}{R_{i\theta}}\to \pi_{ik\theta}$ with probability $1$. Then, since $\hat{\pi}_{r_{ik\theta}}=\frac{\frac{r_{ik\theta}}{R_{i\theta}}R_{i\theta} +\frac{n_{ikt}}{t}t+1}{R_{i\theta}+t+K-1}$ is continuous at $\frac{r_{ik\theta}}{R_{i\theta}}=\pi_{ik\theta}$, $\hat{\pi}_{r_{ik\theta}}(\frac{r_{ik\theta}}{R_{i\theta}})\to \pi_{ik\theta}$ with probability 1 as $R_{i\theta}\to \infty$. Then, as $t\to\infty$, $\hat{\pi}_0(\frac{n_{ikt}}{t})\to\pi_{ik\theta^*}$ with probability $1$ as stated above. Therefore condition 2 holds. When $\pi_{ik\theta^*}=0$, the likelihood ratio goes to infinity, but the private signal $\omega_{it}=k$ will never be received as $\theta^*$ is the ground truth. 
\end{proof}

This immediately results in the following corollary. 

\begin{corollary}\label{cor:ell_dog}
When the agent is certain, i.e., $R_{i\theta}\to\infty$, and $\mathbf{r}_{i\theta}$ is drawn from the distribution $\boldsymbol{\pi}_{i\theta} = \boldsymbol{\pi}_{i\theta^*}$, then the likelihood update of the measurement $\omega_{it+1}=k$ converges to 
\begin{eqnarray}
    \lim_{t\to\infty, R_{i\theta}\to\infty}  \ell_{i\theta}(\mathbf{n}_{it},\omega_{it+1}|\mathbf{r}_{i\theta}) = 1, \ \text{a.s.}
\end{eqnarray}
\end{corollary}

The above lemma and corollary show that modeling with uncertainty results in a likelihood function that varies with time. Furthermore, Lemma \ref{lem:update_lim} condition 2 and Corollary \ref{cor:ell_dog} show that in the certain case, the numerator of the likelihood function is a constant and is modeled in the same manner as the traditional non-Bayesian social learning theory. Thus, the proposed uncertain likelihood ratio translates to a likelihood function that models uncertain and certain conditions based on the amount of prior evidence. 

Therefore, at time $t\ge1$, agent $i$ will combine their neighbors' beliefs of hypothesis $\theta$ at time $t$ and update their belief of $\theta$ using the likelihood update (\ref{eq:el;_def}) of the private signal at time $t+1$ according to (\ref{eq:main_algo}). Then, the agent can interpret hypothesis $\theta$ using the uncertain likelihood ratio test, except $\Lambda_{i\theta}(t)$ is now replaced with the agents belief $\mu_{it}(\theta)$.

\subsection{Asymptotic Behavior on Arbitrary Graphs}

Next, we present the proof of main results in Theorem~\ref{thm:ULR_Con}.  First, we begin by providing three auxiliary lemmas. The first lemma provides a result about the convergence of a product of doubly stochastic matrices provided in \cite{NOU2017}.

\begin{lemma}[Lemma 5 in~\cite{NOU2017}] \label{lem:doubly}
For a stationary doubly stochastic matrix, we have for all $t>0$
\begin{eqnarray}
\left \| \mathbf{A}^t - \frac{1}{m} \mathbf{11}' \right\| \le \sqrt{2}m \lambda^t
\end{eqnarray}
	where $\|\cdot\|$ is the spectral norm, $\lambda= 1-\frac{\eta}{4m^2}$, and $\eta$ is a positive constant s.t. if $A_{ij}>0$, then $A_{ij}\ge \eta$.
\end{lemma}

The above lemma shows that every element of a repeated product of a doubly stochastic matrices will converge to $1/m$. Next, we present the bounds of the likelihood update to show that $\ell_{i\theta}(\omega_{it})$ is bounded.
\begin{lemma} \label{lem:el_b}
	For an uncertain likelihood, i.e., $R_{i\theta}<\infty$, the likelihood update is bounded as follows.
	\begin{eqnarray} \label{eq:el_b}
	\frac{1}{R_{i\theta}+K} \le \ell_{i\theta}(\mathbf{n}_{it-1},\omega_{it}|\mathbf{r}_{i\theta})  \le \max_{k\in\Omega}(r_{ik\theta})+1.
	\end{eqnarray}
\end{lemma}

\begin{proof}
	Consider that the agent $i$ has received $n_{ikt-1}\in[0, t-1]$ private signals for attribute $k$ and $n_{i\bar{k}t-1}=t-1-n_{ikt-1}$ for other signals up to time $t-1$ where $\bar{k} = \boldsymbol{\Omega}\setminus \{k\}$. Then, if agent $i$ receives $\omega_{it}=k$ at time $t$, the log of the likelihood update is
	\begin{eqnarray}\label{eq:lit_n1}
	\log(\ell_{i\theta}(\mathbf{n}_{it-1},k|\mathbf{r}_{i\theta}) = \log\left((r_{ik\theta}+n_{ikt-1}+1)(t+K-1)\right) \nonumber \\
	- \log\left((R_{i\theta}+t+K-1)(n_{ikt-1}+1)\right). \nonumber
	\end{eqnarray}
	The partial derivatives of the update with respect to $n_{ikt-1}$ is
\begin{eqnarray}
\frac{\partial \log(\ell_{i\theta}(\mathbf{n}_{it-1},k|\mathbf{r}_{i\theta}))}{\partial n_{ikt-1}} = \frac{1}{(r_{ik\theta}+n_{ikt-1}+1)} - \frac{1}{(n_{ikt-1}+1)} \nonumber \\
=\frac{-r_{ik\theta}}{(r_{ik\theta}+n_{ikt-1}+1)(n_{ikt-1}+1)} < 0. \nonumber
\end{eqnarray}
	
	Therefore, since the function $\log(\ell_{i\theta}(\mathbf{n}_{it-1},k|\mathbf{r}_{i\theta}))$ is monotonically decreasing with respect to $n_{ikt-1}$, the maximum and minimum occur at $n_{ikt-1}=0$ and $n_{ikt-1}=t-1$, respectively. To maximize the update, setting  $n_{ikt-1}=0$ leads to
\begin{eqnarray}
	\log(\ell_{i\theta}(\mathbf{n}_{it-1},k|\mathbf{r}_{i\theta})) = \log\left((r_{ik\theta}+1)(t+K-1)\right) \nonumber \\ 
	- \log\left((R_{i\theta}+t+K-1)\right) \nonumber
\end{eqnarray}
so that the derivative of the log-update with respect to $t$ is
\begin{equation*}
\frac{d \log(\ell_{i\theta}(\mathbf{n}_{it-1},k|\mathbf{r}_{i\theta}))}{d t} = \frac{R_{i\theta}}{(R_{i\theta}+t+K-1)(t+K-1)} > 0.
\end{equation*}
So the update is maximized by letting $t \to \infty$ so that $\log(\ell_{i\theta}(\mathbf{n}_{it-1},k|\mathbf{r}_{i\theta})) \le \log(r_{ik\theta}+1) \le \log\left(\max_{k\in\Omega}(r_{ik\theta})+1\right)$.
Now to minimize the update, setting $n_{ikt-1}=t-1$ leads to
\begin{eqnarray}
	\log(\ell_{i\theta}(\mathbf{n}_{it-1},k|\mathbf{r}_{i\theta})) = \log\left((r_{ik\theta}+t)(t+K-1)\right) \nonumber \\ 
	- \log\left((R_{i\theta}+t+K-1)t\right). \nonumber
\end{eqnarray}
Now $\log(t+K-1)-\log(t) \ge 0$ and $\log(r_{ik\theta}+t)-\log(R_{i\theta}+t+K-1)$ is minimized over $t \ge 1$ at $t=1$ so that $\log(r_{ik\theta}+t)-\log(R_{i\theta}+t+K-1) \ge \log(r_{ik\theta}+1)-\log(R_{i\theta}+K) \ge -\log(R_{i\theta}+K)$. Thus, $\log(\ell_{i\theta}(\mathbf{n}_{it-1},k|\mathbf{r}_{i\theta})) \ge -\log(R_{i\theta}+K)$ for all $k \in \Omega$.
\end{proof}
Finally, we recall Lemma 3.1 from \cite{ram10}, which provides a convergence property of scalar sequences. 
\begin{lemma}[Lemma $3.1$ in \cite{ram10}]\label{lemm:ram}
		Let $\{\gamma_k \}$ be a scalar sequence. If  $\lim_{k \to \infty} \gamma_k = \gamma$ and $0\leq \beta \leq 1$, then $\lim_{k\to \infty} \sum_{l=0}^{k} \beta^{k{-}l} \gamma_l = \frac{\gamma}{1{-}\beta}$.
\end{lemma}

\begin{proof}[Proof of Theorem \ref{thm:ULR_Con}]
With the above lemmas stated, we can now prove Theorem \ref{thm:ULR_Con}. First, we prove that the beliefs converge to the $m$th root of the product of uncertain likelihood ratios, i.e.,
\begin{eqnarray} \label{eq:con_diff}
\lim_{t\rightarrow \infty} \ \ \left\| \log(\boldsymbol{\mu}_{t}(\theta)) - \log\left( (\prod_{i=1}^m \Lambda_{i\theta}(t))^{\frac{1}{m}}\right) \mathbf{1}  \right\| = 0,
\end{eqnarray}
where for vectors $\|\cdot\|$ is the standard 2-norm. Thus, since
\begin{equation}
\begin{split}
&\log(\boldsymbol{\mu}_{t}(\theta)) = \sum_{\tau = 1}^{t} \mathbf{A}^{t-\tau} \log\left(\boldsymbol{\ell}_\theta(\boldsymbol{\omega}_\tau)\right), \ \text{and} \label{eq:log_belief} \\
&\log\left((\prod_{i=1}^m \Lambda_{i\theta}(t))^\frac{1}{m} \right)\mathbf{1} =   \frac{1}{m}\mathbf{11'} \sum_{\tau=1}^t \log(\boldsymbol{\ell}_\theta(\boldsymbol{\omega}_\tau)),
\end{split}
\end{equation}
where with a slight abuse of notation, $\boldsymbol{\ell}_\theta(\boldsymbol{\omega}_{\tau}) = [\ell_{1\theta}(\mathbf{n}_{1\tau-1},\omega_{1\tau}|\mathbf{r}_{1\theta}),...,\ell_{m\theta}(\mathbf{n}_{m\tau-1},\omega_{m\tau}|\mathbf{r}_{m\theta})]'$, (\ref{eq:con_diff}) can be rewritten as

\begin{eqnarray} \label{eq:norm_LL}
\left\|\log(\boldsymbol{\mu}_{t}(\theta))-\log\left((\prod_{i=1}^m \Lambda_{i\theta}(t))^\frac{1}{m} \right)\mathbf{1}  \right\|  \le \nonumber \\ \sum_{\tau=1}^{t}  \left\|\mathbf{A}^{t-\tau}-\frac{1}{m}\mathbf{11}'  \right\|\left\|\log\left(\boldsymbol{\ell}_\theta(\boldsymbol{\omega}_\tau)\right)\right\| \le \nonumber \\
\sqrt{2}m\left(\sum_{\tau=0}^{t} \lambda^{t-\tau}  \left\|\log\left(\boldsymbol{\ell}_\theta(\boldsymbol{\omega}_\tau)\right)\right\| - \lambda^t \left\|\log\left(\boldsymbol{\ell}_\theta(\boldsymbol{\omega}_0)\right)\right\| \right),
\end{eqnarray}
where (\ref{eq:norm_LL}) follows from Lemma~\ref{lem:doubly}. Furthermore, since $\lim_{t\to\infty} \left\|\log\left(\boldsymbol{\ell}_\theta(\boldsymbol{\omega}_0)\right)\right\|=0$ a.s. from Lemma~\ref{lem:update_lim}, then 

\begin{eqnarray}
\lim_{t\to\infty} \sum_{\tau=0}^{t} \lambda^{t-\tau}  \left\|\log\left(\boldsymbol{\ell}_\theta(\boldsymbol{\omega}_\tau)\right)\right\| = 0 \nonumber
\end{eqnarray}
a.s. from Lemma~\ref{lemm:ram}. Finally, since $\lambda<1$ and $\left\|\log\left(\boldsymbol{\ell}_\theta(\boldsymbol{\omega}_0)\right)\right\|$ is bounded according to Lemma~\ref{lem:el_b}
\begin{eqnarray}
    \lim_{t\to\infty} \lambda^t \left\|\log\left(\boldsymbol{\ell}_\theta(\boldsymbol{\omega}_0)\right)\right\| = 0 \ \ \ \text{a.s..} \nonumber
\end{eqnarray}
Then, by the continuity of the logarithmic function, this implies that $\lim_{t\to\infty} \boldsymbol{\mu}_t(\theta)/ \left( \prod_{j=1}^m \Lambda_{j\theta}(t)\right)^{1/m} = 1$ a.s. and the desired result is achieved. 
\end{proof}

\subsection{Learning with Certain Likelihoods} \label{sec:Con_Dog}

Next, we present the results for when the agents are certain, i.e., $R_{i\theta}\to\infty$. First, we will consider the scenario when hypothesis $\theta$ is the ground truth for all agents, i.e., $\boldsymbol{\pi}_{i\theta}=\boldsymbol{\pi}_{i\theta^*}$ $\forall i\in \mathbf{M}$. Then, we will present the condition when hypothesis $\theta$ is not the ground truth for at least one agent $i$, i.e., $\boldsymbol{\pi}_{i\theta}\ne\boldsymbol{\pi}_{i\theta^*}$.

\begin{corollary} \label{lem:LL_Dog_Inf}
Let Assumptions \ref{assum:graph} and \ref{assum:inital_beliefs} hold and $\boldsymbol{\pi}_{i\theta}=\boldsymbol{\pi}_{i\theta^*}$ $\forall i\in \mathbf{M}$. Then, the beliefs generated using the update rule (\ref{eq:main_algo}) with infinite evidence diverge to the following.
\begin{eqnarray}
\lim_{t\rightarrow \infty} \mu_{it}(\theta) = \infty, \ \text{a.s.}
\end{eqnarray}
\end{corollary}
\begin{proof}
By Corollary~\ref{cor:ell_dog}, $\lim_{t\to\infty} \ell_{i\theta}(\mathbf{n}_{it-1},\omega_{it}|\mathbf{r}_{i\theta}) = 1$ a.s..  As a result, the proof of Theorem~\ref{thm:ULR_Con} still applies and $\mu_{it}(\theta) = (\prod_{i=1}^m \Lambda_{i\theta}(t))^\frac{1}{m}$ as $t \to \infty$ with probability 1.  Now by Lemma~\ref{lem:dog_ULR}, $\Lambda_{i\theta}(t) = \infty$ for each $i$ as $t \to \infty$ and $R_{i\theta}\to\infty$. Thus, the geometric mean is also diverging to $\infty$ a.s..
\end{proof}

\begin{lemma} \label{lem:LL_dog}
    Let Assumption \ref{assum:graph} and \ref{assum:inital_beliefs} hold and at least one agent $i\in\mathbf{M}$ has a set of probabilities s.t. $\boldsymbol{\pi}_{i\theta}\ne\boldsymbol{\pi}_{i\theta^*}$. Then, the beliefs generated by the update rule (\ref{eq:main_algo}) allow for learning, i.e., they converge in probability to 
    \begin{eqnarray}
    \mu_{it}(\theta) \overset{P}{\to} 0. 
    \end{eqnarray}
\end{lemma}

Before proving Lemma~\ref{lem:LL_dog}, we must first present the following lemma which provides an upper bound of the certain likelihood update.

\begin{lemma} \label{lem:ell_dog_bound}
For a finite time $t$, the certain likelihood update is bounded above by
\begin{eqnarray} \label{eq:ell_dog_bound}
\ell_{i\theta}(\mathbf{n}_{it-1},\omega_{it}|\mathbf{r}_{i\theta}) \le (t+K-1) < \infty
\end{eqnarray}
\end{lemma}
\begin{proof}
First, by inspection of (\ref{eq:ell_update}) for the certain condition such that $\ell_{i\theta}(\mathbf{n}_{it-1},\omega_{it}|\mathbf{r}_{i\theta}) = \pi_{ik\theta} \frac{t+K-1}{n_{ikt-1}+1} $ , it is clear that the maximum occurs when an attribute $k\in\Omega$ has not been received up to time $t-1$. In other words, the term $\frac{t+K-1}{n_{ikt-1}+1}$ is maximized when $n_{ikt-1}=0$, resulting in the likelihood update being bounded by $(t+K-1)$ because $\pi_{ik\theta}\le 1$. For any finite value of $t$ this value is the highest possible value for the update.
\end{proof}

Now that the likelihood update is shown to be bounded by a finite value for finite $t$, we can now prove Lemma~\ref{lem:LL_dog}.

\begin{proof}[Proof of Lemma~\ref{lem:LL_dog}]
Starting with (\ref{eq:log_belief}) the log-beliefs $\mu_{it}(\theta)$ can be written as
\begin{eqnarray} \label{eq:dog-belief}
\log(\boldsymbol{\mu}_{t}(\theta)) &=&
\sum_{\tau = 1}^{T} \mathbf{A}^{t-\tau} \log\left(\boldsymbol{\ell}_\theta(\boldsymbol{\omega}_\tau)\right)  \nonumber \\
& &+ \sum_{\tau = T+1}^{t} \left (\mathbf{A}^{t-\tau}-\frac{1}{m}\mathbf{11}'\right) \log\left(\boldsymbol{\ell}_\theta(\boldsymbol{\omega}_\tau)\right) \nonumber \\ & &  + \frac{1}{m}\sum_{\tau = T+1}^{t}\mathbf{11}'\log\left(\boldsymbol{\ell}_\theta(\boldsymbol{\omega}_\tau)\right).
\end{eqnarray}
Now because $\mathbf{A}$ is doubly stochastic, $\|\mathbf{A}\|=1$ and the norm of the first term in the right hand side of (\ref{eq:dog-belief}) is bounded by
\begin{eqnarray*}
\left\|\sum_{\tau = 1}^{T} \mathbf{A}^{t-\tau} \log\left(\boldsymbol{\ell}_\theta(\boldsymbol{\omega}_\tau)\right) \right\| &\leq& \sum_{\tau = 1}^{T} \left \| \log\left(\boldsymbol{\ell}_\theta(\boldsymbol{\omega}_\tau)\right) \right\|\\
& \leq &\sum_{\tau = 1}^{T} \sqrt{m}(\tau+K-1) \nonumber \\  &=& \sqrt{m}T\left(\frac{1}{2}(T+1)+(K-1)\right),
\end{eqnarray*}
where the second line is the result of the upper bound for the possible update value given in Lemma~\ref{lem:ell_dog_bound}. As long as $T$ is finite this first term is finite.

By Lemma~\ref{lem:update_lim}, $\log\left(\ell_{i\theta}(\mathbf{n}_{i\tau-1},\omega_{\tau}|\mathbf{r}_{i\theta})\right) \to \log\left(\frac{\pi_{ik\theta}}{\pi_{ik\theta^*}}\right)$ a.s., and so for any $\epsilon>0$ and $\delta>0$ there exist a finite value $T$ such that $|\log\left(\ell_{i\theta}(\mathbf{n}_{i\tau-1},\omega_{\tau}|\mathbf{r}_{i\theta})\right)-\log\left(\frac{\pi_{ik\theta}}{\pi_{ik\theta^*}}\right)|<\epsilon$ with probability greater than $1-\delta$. Thus the second term on the right hand side of (\ref{eq:dog-belief}) with probability greater than $1-\delta$ is bounded by
\begin{eqnarray*}
& &\left\| \sum_{\tau = T+1}^{t} \left (\mathbf{A}^{t-\tau}-\frac{1}{m}\mathbf{11}'\right) \log\left(\boldsymbol{\ell}_\theta(\boldsymbol{\omega}_\tau)\right) \right\| \nonumber \\ &\leq& \sum_{\tau = T+1}^{t}  \left\|\mathbf{A}^{t-\tau}-\frac{1}{m}\mathbf{11}' \right\| \left\| \log\left(\boldsymbol{\ell}_\theta(\boldsymbol{\omega}_\tau)\right)\right\|\\
&\leq &\sqrt{2}m \left( \sum_{\tau=T+1}^t \lambda^{t-\tau} \right) (L+\epsilon) \leq  \frac{\sqrt{2}m}{(1-\lambda)} (L+\epsilon),
\end{eqnarray*}
where $L = \max_{i,k\in\Omega_i^*}  \left|\log\left(\frac{\pi_{ik\theta}}{\pi_{ik\theta^*}}\right) \right|$ is the largest converged value that is realizable, i.e, $\Omega_i^*$ is the set of all $k$ values such that $\pi_{ik\theta^*} >0$.  Because $L$ is finite, the second term in (\ref{eq:dog-belief}) is also finite.

Finally, each element for the third term on the right hand side of (\ref{eq:dog-belief}) can be reexpressed as
\begin{eqnarray}
& & \frac{1}{m} \sum_{\tau = T+1}^t \sum_{i=1}^m \log\left(\ell_{i\theta}(\mathbf{n}_{i\tau-1},\omega_{i\tau}|\mathbf{r}_{i\theta})\right) \nonumber \\ &=& \frac{1}{m} \sum_{\tau=T+1}^t \sum_{i=1}^m \left( \log \left(\frac{\pi_{i\omega_{i\tau}\theta}}{\pi_{i\omega_{i\tau}\theta^*}}\right) + e_{i\tau}\right), \nonumber\\
&\leq& (t-T) \left( \frac{1}{(t-T)}\left( \sum_{\tau=T+1}^t  \frac{1}{m} \sum_{i=1}^m \log \left(\frac{\pi_{i\omega_{i\tau}\theta}}{\pi_{i\omega_{i\tau}\theta^*}}\right)\right) + \epsilon \right), \nonumber \\
\label{eq:fullyconnectedlogupdatebound}
\end{eqnarray}
where $e_{i\tau} = \log\left(\ell_{i\theta}(\mathbf{n}_{i\tau-1},\omega_{i\tau}|\mathbf{r}_{i\theta})\right)-\log \left(\frac{\pi_{i\omega_{i\tau}\theta}}{\pi_{i\omega_{i\tau}\theta^*}}\right)$ is the error and $|e_{i\tau}| \le \epsilon$, which leads to the second line. Due the strong law of large numbers, the bound for the third term converges with probability one to
\begin{equation}
(t-T) \left(-\frac{1}{m} \sum_{i=1}^m D_{KL}(\boldsymbol{\pi}_{i\theta^*}||\boldsymbol{\pi}_{i\theta})+\epsilon\right).
\end{equation}
In other words, for $t$ sufficiently large with probability $1-\delta$
\begin{eqnarray}
\log(\mu_{it}(\theta)) &\leq& \sqrt{m}\frac{T}{2}(T+1)+T(K-1)+\frac{\sqrt{2}m}{(1-\lambda)}(L+\epsilon)\nonumber \\ & & +(t-T)\left(-\frac{1}{m}\sum_{i=1}^m D_{KL}(\boldsymbol{\pi}_{i\theta^*}||\boldsymbol{\pi}_{i\theta}) + 2\epsilon\right) \nonumber 
\end{eqnarray}
Since $\frac{1}{m} \sum_{i=1}^m D_{KL}(\boldsymbol{\pi}_{i\theta^*}||\boldsymbol{\pi}_{i\theta})>0$ as $\boldsymbol{\pi}_{\theta^*}\ne \boldsymbol{\pi}_\theta$, and $\epsilon$ can be made smaller by making $T$  larger, the upper bound is diverging to $-\infty$ as $t$ increases.  Thus, the log-belief is diverging to $-\infty$ as $t \to \infty$.  Because the exponential is continuous, the beliefs converge in probability to zero. 
\end{proof}

Corollary~\ref{lem:LL_Dog_Inf} and Lemma~\ref{lem:LL_dog} show that in order for the agents to learn the ground truth precisely, all of the agents must have certain probability distributions that match the ground truth exactly. While if a single agent disagrees, then the beliefs will converge to $0$. Therefore, this result is consistent with the traditional non-Bayesian social learning literature except that the hypothesis that matches the ground truth diverges to infinity instead of converging to $1$. Thus, the design of the uncertain likelihood ratio still preserves the consensus result while allowing the agents to consider uncertain scenarios.

After expanding the beta functions and applying Stirling's approximation, it can be shown that the certain likelihood ratio for large $t$ behaves as 
\begin{eqnarray} \label{eq:th_lambda}
\Lambda_{i\theta}(t) = C t^\alpha e^{-tD_{KL}(\boldsymbol{\pi}_{i\theta^*}||\boldsymbol{\pi}_{i\theta})},
\end{eqnarray}
where $C$ and $\alpha$ are constants. Note that in the centalized uncertain likelihood ratio is the product of the individual uncertain likelihood ratios. Without any divergence between $\boldsymbol{\pi}_{i\theta}$ and $\boldsymbol{\pi}_{i\theta^*}$ for all agents, the uncertain likelihood ratio goes to infinity sub-exponentially as $t^\alpha$. It only takes any divergence between $\boldsymbol{\pi}_{i\theta}$ and $\boldsymbol{\pi}_{i\theta^*}$ at a single agent to drive the centralized uncertain likelihood ratio to zero as the decay to zero is exponential. Essentially, a hypothesis $\theta$ that is consistent with the observations can never be declared as the absolute ground truth as any new certain agent whose model for that hypothesis is inconsistent with their observation would drive the uncertain likelihood ratio to zero. Rather, one can only state that the hypothesis is consistent with the ground truth as no counter example has been observed. On the other hand, once a counter example is found by any agent, one can state unequivocally that the hypothesis is not the ground truth. No finite number of agents such that $\boldsymbol{\pi}_{i\theta} = \boldsymbol{\pi}_{i\theta^*}$ can drive the belief to be non-zero.

For the more general uncertain case, the updates $\ell_{i\theta}(\omega_{it})$ as given by (\ref{eq:ell_update}) begin as ratios of the expected value of $\boldsymbol{\pi}_{i\theta}$ based upon the prior evidence $\mathbf{r}_{i\theta}$ over that based upon the observations $\mathbf{n}_{it}$ .  As time evolves, the numerator of the ratio transitions from an estimate of $\boldsymbol{\pi}_{i\theta}$ to that of $\boldsymbol{\pi}_{i\theta^*}$.  On the other hand, the denominator is going to an estimate of $\boldsymbol{\pi}_{i\theta^*}$.  The larger the amount of prior evidence $R_{i\theta}$, the longer it takes for the transition to occur.  Before the transition, the uncertain likelihood ratio behaves like the certain case.  After the transition, the updates converge to one, which cause the uncertain likelihood ratio to level out.  If $\theta \ne \theta^*$, whether or not the uncertain likelihood ratio converges to a value larger or less than one depends on whether or not the divergence between the $\pi$'s is able to overwhelm the $t^\alpha$ growth before the updates become close to one.  This in turn depends on the amount of prior evidence.  Less prior evidence means that $\theta$ may not be distinguished from $\theta^*$ given the precision of the evidence. The simulations in Section~\ref{sec:SIM} will demonstrate these properties.

\section{The Effects of DeGroot Aggregation for Uncertain Models} \label{sec:DG_Update}
Next, we will consider a DeGroot-style update rule and present the effects of the beliefs with uncertain likelihood models. The DeGroot-style update rule consists of taking the weighted arithmetic average of the agents prior beliefs instead of the geometric average. Thus, the DeGroot-style update rule with uncertain likelihood models is defined as,
\begin{eqnarray} \label{eq:DG_algo}
    \mu_{it+1}(\theta) = \ell_{i\theta}(\mathbf{n}_{it}, \omega_{it+1}|\mathbf{r}_{i\theta})\sum_{j\in\mathbf{M}^i} [\mathbf{A}]_{ij} \mu_{jt}(\theta).
\end{eqnarray}

First, let us consider the asymptotic properties of the beliefs generated using the update rule (\ref{eq:DG_algo}) with a finite amount of prior evidence.
\begin{lemma} \label{thm:DG_finite}
Let Assumptions \ref{assum:graph}, \ref{assum:uncertain_distributioin}, and \ref{assum:inital_beliefs} hold. Then, the beliefs generated using the update rule (\ref{eq:DG_algo}) have the following property with probability 1:
\begin{eqnarray} \label{eq:CES_LL}
\lim_{t\to \infty} \mu_{it}(\theta)\ge \left( \prod_{i=1}^m \tilde{\Lambda}_{i\theta} \right)^{\frac{1}{m}}.
\end{eqnarray}
\end{lemma}

\begin{proof}
    To prove this, we will first compare the beliefs generated from the update rule (\ref{eq:DG_algo}), denoted $\boldsymbol{\mu}_t^{[DG]}(\theta)$, with the beliefs generated for the update rule (\ref{eq:main_algo}),  denoted $\boldsymbol{\mu}_t^{[LL]}(\theta)$. Then, by induction, we have the following. At $t=0$, the agents beliefs are initialized to the same value, $\boldsymbol{\mu}_0^{[DG]}(\theta)=\boldsymbol{\mu}_0^{[LL]}(\theta)=\mathbf{1}$ and $\boldsymbol{\mu}_0^{[DG]}(\theta)\ge\boldsymbol{\mu}_0^{[LL]}(\theta)$ is true. Given that $\boldsymbol{\mu}_{t-1}^{[DG]}(\theta)\ge\boldsymbol{\mu}_{t-1}^{[LL]}(\theta)$ is true for time $t-1$, the log of the beliefs from the DeGroot and LL rules ag time $t$ respectively becomes
    \begin{eqnarray}
    \log(\mu_{it}^{[DG]}(\theta))  &=&  \log(\ell_{i\theta}(\mathbf{n}_{it}, \omega_{it+1}|\mathbf{r}_{i\theta}) ) \nonumber \\ & & + \log\left(\sum_{j=1}^m [\mathbf{A}]_{ij} \mu_{jt-1}^{[DG]}(\theta)\right), \nonumber\\
    \log(\mu_{it}^{[LL]}(\theta))  &=&  \log(\ell_{i\theta}(\mathbf{n}_{it}, \omega_{it+1}|\mathbf{r}_{i\theta}) ) \nonumber \\ & &+ \sum_{j=1}^m [\mathbf{A}]_{ij} \log\left(\mu_{jt-1}^{[LL]}(\theta)\right) \nonumber
    \end{eqnarray}
    Using Jensen's inequality, $\log(\sum_{j=1}^m [\mathbf{A}]_{ij}\mu_{jt-1}^{[DG]}(\theta)) \ge \sum_{j=1}^m [\mathbf{A}]_{ij} \log(\mu_{jt-1}^{[DG]}(\theta))$ since the logarithm is a concave function. Since $\sum_{j=1}^m [\mathbf{A}]_{ij}\log(\mu_{jt-1}^{[DG]}(\theta)) \ge \sum_{j=1}^m [\mathbf{A}]_{ij} \log(\mu_{jt-1}^{[LL]}(\theta))$, $\boldsymbol{\mu}_{t}^{[DG]}(\theta)\ge\boldsymbol{\mu}_{t}^{[LL]}(\theta)$.  By induction, $\boldsymbol{\mu}_{t}^{[DG]}(\theta)\ge\boldsymbol{\mu}_{t}^{[LL]}(\theta)$ is true $\forall t\ge 0$, and asymptotically we can say that
    \begin{eqnarray}
    \lim_{t\rightarrow \infty} \mu_{it}^{[DG]}(\theta) \ge \lim_{t\rightarrow \infty} \mu_{it}^{[LL]}(\theta) = \left( \prod_{i=1}^m \tilde{\Lambda}_{i\theta} \right)^{\frac{1}{m}} \nonumber
    \end{eqnarray}
    with probability 1.
\end{proof}

Lemma \ref{thm:DG_finite} shows that the beliefs generated from the DeGroot-style update rule will always be greater than or equal to the $m$th root of the centralized uncertain likelihood ratio. This means that the interpretation of the beliefs using the update rule (\ref{eq:main_algo}) and the DeGroot rule (\ref{eq:DG_algo}) are not the same.  Nevertheless the simulations in Section~\ref{sec:SIM} demonstrates that the DeGroot rule reaches consensus but is non-commutative because the order in which the private signals are received affects where the belief converges. Thus, a further understanding of the beliefs point of convergence is necessary to identify thresholds that allow for the use of the uncertain likelihood ratio test. This will be studied as a future work.

The certain likelihood conditions presented next indicate that the DeGroot rule still enables learning. Additionally, we derive the beliefs asymptotic convergence rate for a fully connected network and show that learning with the update rule (\ref{eq:DG_algo}) is slower than learning with (\ref{eq:main_algo}). First, noting the result of the uncertain DeGroot-style update rule, we can conclude the following corollary.

\begin{corollary}\label{cor:DG_dog}
    Let Assumptions \ref{assum:graph} and \ref{assum:inital_beliefs} hold and $\boldsymbol{\pi}_{i\theta}=\boldsymbol{\pi}_{i\theta^*}$ $\forall i\in \mathbf{M}$. Then, the beliefs generated using the update rule (\ref{eq:DG_algo}) and infinite evidence diverge to the following.
\begin{eqnarray}
\lim_{t\rightarrow \infty} \mu_{it}(\theta)= \infty, \ \text{a.s.}
\end{eqnarray}
\end{corollary}
\begin{proof}
    This can be directly seen from Lemma \ref{thm:DG_finite} and Corollary \ref{lem:LL_Dog_Inf}.
\end{proof}

Next, we will derive the point of convergence when at least one agent $i$ has a certain set of probabilities s.t. $\boldsymbol{\pi}_{i\theta}\ne \boldsymbol{\pi}_{i\theta^*}$. First, we provide the following lemma that describes the properties of the beliefs updated using the DeGroot-style learning rule for a fully connected network.
\begin{lemma} \label{lem:DG_rate}
Let Assumption \ref{assum:inital_beliefs} hold, the network graph be fully connected, i.e. $\mathbf{A}=\frac{1}{m}\mathbf{11}'$, and there exists a $\theta$ s.t. $\boldsymbol{\pi}_{i\theta}\ne \boldsymbol{\pi}_{i\theta^*}$ for at least one agent~$i$. Then, the beliefs generated by the update rule (\ref{eq:DG_algo}) with infinite evidence asymptotically convergence to zero at a geometric rate determined by the Centralized Average (CA) divergence, i.e., for all $i \in \mathbf{M}$
\begin{eqnarray}
\lim_{t \to \infty} \frac{1}{t} \mu_{it}(\theta) = -D_{CA}(\boldsymbol{\Pi}_{\theta^*}||\boldsymbol{\Pi}_{\theta}),
	\end{eqnarray}
	where 
	\begin{eqnarray}
	D_{CA}(\boldsymbol{\Pi}_{\theta^*}||\boldsymbol{\Pi}_{\theta}) =  -\sum_{k_1=1}^K \cdots \sum_{k_m=1}^K \pi_{1k_1 \theta^*}\cdots \pi_{mk_m \theta^*} \nonumber \\ \cdot \log\left(\frac{1}{m}\left(  \frac{\pi_{1k_1\theta}}{\pi_{1k_1\theta^*}}+\cdots + \frac{\pi_{mk_m\theta}}{\pi_{mk_m\theta^*}}\right)\right)  
	\end{eqnarray}
	and $\boldsymbol{\Pi}_{\theta}=\{\boldsymbol{\pi}_{i\theta}\}_{\forall i\in \mathbf{M}}$ is the set of probabilities of all agents.
\end{lemma}

\begin{proof}
    First, from Lemma \ref{lem:update_lim} condition 2, the likelihood updates converge to the ratio of the probabilities for $\theta$ and $\theta^*$. Since the logarithm and average operations are continuous, we know that for any $\epsilon>0$ and $\delta>0$, there exists a finite $T$ s.t. for $t>T$ the log average likelihood update is bounded as
     \begin{equation*}
     \left | \log\left(\frac{1}{m} \sum_{i=1}^m \ell_{i\theta}(\mathbf{n}_{it-1}\omega_{it}|\mathbf{r}_{i\theta})\right) - \log\left( \frac{1}{m} \sum_{i=1}^m \frac{\pi_{i\omega_{it}\theta}}{\pi_{i\omega_{it}\theta^*}} \right) \right | \leq \epsilon
     \end{equation*}
     with probability at least $1-\delta$. Also, we know that $\boldsymbol{\mu}_T(\theta)$ is bounded since $\ell_{i\theta}(\mathbf{n}_{it-1},\omega_{it}|\mathbf{r}_{i\theta})$ is bound by Lemma \ref{lem:ell_dog_bound} and converging to within $\ell_{i\theta}(\mathbf{n}_{it-1},k|\mathbf{r}_{i\theta})<\frac{\pi_{ik\theta}}{\pi_{ik\theta^*}}+\epsilon$ with probability at least $1-\delta$. Now, the beliefs generated using the update rule (\ref{eq:DG_algo}) at times times $t$  and $T$ are related as 
    \begin{eqnarray}\label{eq:dg_bel_con}
\boldsymbol{\mu}_{t}(\theta) &=& \mathbf{L}_\theta(\boldsymbol{\omega}_{t}) \mathbf{A} \mathbf{L}_\theta(\boldsymbol{\omega}_{t-1}) \cdots \mathbf{A}\mathbf{L}_\theta(\boldsymbol{\omega}_{T+1}) \mathbf{A} \boldsymbol{\mu}_T(\theta)  \nonumber \\
&=& \mathbf{L}_\theta(\boldsymbol{\omega}_{t}) \frac{1}{m}\mathbf{11'} \mathbf{L}_\theta(\boldsymbol{\omega}_{t-1}) \cdots \frac{1}{m}\mathbf{11'}\mathbf{L}_\theta(\boldsymbol{\omega}_{T+1}) \nonumber \\ & &\cdot \left( \frac{1}{m} \sum_{i=1}^m \mu_{iT}(\theta)\right) \mathbf{1}  \nonumber \\
&= & \prod_{\tau = T+1}^t \left( \frac{1}{m} \sum_{i=1}^m \ell_{i\theta}(\mathbf{n}_{i\tau-1}\omega_{i\tau}|\mathbf{r}_{i\theta}) \right) \left( \frac{1}{m} \sum_{i=1}^m \mu_{iT}(\theta)\right) \nonumber \\
\end{eqnarray}
where $\mathbf{L}_\theta(\boldsymbol{\omega}_\tau)=  diag(\ell_{1\theta}(\mathbf{n}_{1\tau-1},\omega_{1\tau}|\mathbf{r}_{1\theta}),...,$ $\ell_{m\theta}(\mathbf{n}_{m\tau-1},\omega_{m\tau}|\mathbf{r}_{m\theta}))$. We then take the logarithm of both sides of the above equation and use the knowledge that the log-updates are bounded in probability to determine that the bounds with probability at least $1-\delta$ for the log-beliefs are
\begin{equation*}
G(t;T)-(t-T)\epsilon \leq log\left(\boldsymbol{\mu}_t(\theta)\right) \leq G(t;T)+(t-T)\epsilon,
\end{equation*}
where
\begin{equation*}
G(t;T) = \log\left( \frac{1}{m} \sum_{i=1}^m \mu_{iT}(\theta)\right) + \sum_{\tau=T+1}^t \log\left(\frac{1}{m} \sum_{i=1}^m \frac{\pi_{i\omega_{i\tau}\theta}}{\pi_{i\omega_{i\tau}\theta^*}}\right).
\end{equation*}
Note that the first term $G(t,T)$ is finite and constant with respect to $t$.  Using the law of large numbers the
asymptotic convergence rate is bounded with probability at least $1-\delta$ as
\begin{equation*}
D_{CA}(\boldsymbol{\Pi}_{\theta^*}||\boldsymbol{\Pi}_{\theta})-\epsilon \leq -\lim_{t\rightarrow \infty} \frac{1}{t} \log\left(\boldsymbol{\mu}_{t}(\theta)\right) \leq D_{CA}(\boldsymbol{\Pi}_{\theta^*}||\boldsymbol{\Pi}_{\theta})+\epsilon.
\end{equation*}
Note that $\epsilon$ can be made arbitrarily small by setting $T$ larger.  Thus, the convergence rate converges in probability to 
\begin{equation*}
\lim_{t\rightarrow \infty} \frac{1}{t} \log\left(\boldsymbol{\mu}_{t}(\theta)\right) = -D_{CA}(\boldsymbol{\Pi}_{\theta^*}||\boldsymbol{\Pi}_{\theta}).
\end{equation*}
\end{proof}

This shows that even for the DeGroot-style rule, any divergence between $\boldsymbol{\pi}_{i\theta}$ and $\boldsymbol{\pi}_{i\theta^*}$ causes the beliefs to decrease at a rate larger than the sub-exponential growth rate. This is state formally in the following corollary.

\begin{corollary} \label{cor:DG_dog2}
    Let Assumption \ref{assum:inital_beliefs} hold and the network graph be fully connected, i.e., $\mathbf{A}=\frac{1}{m}\mathbf{11'}$, and at least one agent $i$ has a set of probabilities s.t. $\boldsymbol{\pi}_{i\theta}\ne \boldsymbol{\pi}_{i\theta^*}$. Then, the beliefs generated by the update rule (\ref{eq:DG_algo}) with infinite evidence allows for learning, i.e., they converge in probability to 
    \begin{eqnarray}
    \lim_{t\to\infty} \mu_{it}(\theta) = 0.
    \end{eqnarray}
\end{corollary}

Now, let us compare this result to a network updating their beliefs using the log-linear rule (\ref{eq:main_algo}) in the following lemma.

\begin{lemma} \label{thm:comp_rate}
Assuming a network with a doubly stochastic aperiodic matrix $\mathbf{A}$ and a certain set of probabilities such that there exists a $\theta$ s.t. $\boldsymbol{\pi}_{i\theta}\ne \boldsymbol{\pi}_{i\theta^*}$ for at least one agent $i$, the log-linear beliefs (\ref{eq:main_algo}) converge in probability to zero at a geometric rate determined by the average Kullback-Leibler divergence, i.e.,
\begin{equation}
\lim_{t \to \infty} \frac{1}{t} \log\left(\boldsymbol{\mu}_t(\theta)\right) = -\frac{1}{m} \sum_{i=1}^m D_{KL}(\boldsymbol{\pi}_{i\theta^*}||\boldsymbol{\pi}_{i\theta}).
\end{equation}
for $i\in \mathbf{M}$. Furthermore, this convergence rate is faster than that of the DeGroot rule (\ref{eq:DG_algo}) for a fully connected graph where $\mathbf{A}=\frac{1}{m}\mathbf{11}'$, i.e.,
\begin{eqnarray} \label{eq:cov_rates}
	\frac{1}{m}\sum_{i=1}^m D_{KL}(\boldsymbol{\pi}_{i\theta^*}||\boldsymbol{\pi}_{i\theta})\ge D_{CA}(\boldsymbol{\Pi}_{\theta^*}||\boldsymbol{\Pi}_{\theta})\ge 0.
	\end{eqnarray}
\end{lemma}

\begin{proof}
    The proof of Lemma~\ref{lem:LL_dog} provides the starting point to prove the first part of this theorem.
The log belief at time $t$ is expressed by (\ref{eq:dog-belief}).  For any $\epsilon>0$ and $\delta>0$ there exists a value of $T$ such that the first two terms on the right side of (\ref{eq:dog-belief}) are constant with respect to $t$ and finite with probability at least $1-\delta$.  The upper bound for third term is given by (\ref{eq:fullyconnectedlogupdatebound}).  By the same argument to get to this upper bound, it is clear that the lower bound can be given by replacing $\epsilon$ with $-\epsilon$, and thus with probability at least $1-\delta$,
\begin{eqnarray}
\left| \log\left(\mu_{it}(\theta)\right) - C + \left(  \frac{1}{m} \sum_{i=1}^m  \sum_{\tau=T+1}^t\log \left(\frac{\pi_{i\omega_{i\tau}\theta}}{\pi_{i\omega_{i\tau}\theta^*}}\right)\right) \right | \leq \epsilon
\end{eqnarray}
for any $i \in \mathbf{M}$ where $C$ represents the finite constant incorporating the first two terms in (\ref{eq:dog-belief}).
As $t \to \infty$, the law of large numbers leads to the bounds for convergence rate as
\begin{eqnarray}
\frac{1}{m} \sum_{i=1}^m D_{KL}(\boldsymbol{\pi}_{i\theta^*}||\boldsymbol{\pi}_{i\theta})-\epsilon \leq -\lim_{t\to \infty} \log\left(\mu_{it}(\theta) \right) \nonumber \\ \leq \frac{1}{m} \sum_{i=1}^m D_{KL}(\boldsymbol{\pi}_{i\theta^*}||\boldsymbol{\pi}_{i\theta})+\epsilon.
\end{eqnarray}
Note that $\epsilon$ can be made arbitrarily small by increasing the value of $T$ in (\ref{eq:dog-belief}); thus proving the first part of the theorem.

Next, we will prove (\ref{eq:cov_rates}). First, we prove that the CA divergence is non-negative using Jensen's inequality as follows:
\begin{eqnarray}
D_{CA}(\boldsymbol{\Pi}_{\theta^*}||\boldsymbol{\Pi}_{\theta})&=&-E^{\theta^*}\left[ \log \left( \frac{1}{m} \sum_{i=1}^m \frac{\boldsymbol{\pi}_{i\theta}}{\boldsymbol{\pi}_{i\theta^*}}\right)\right], \nonumber \\
& \ge & \log\left(\frac{1}{m}\sum_{i=1}^m E^{\theta^*} \left[\frac{\boldsymbol{\pi}_{i\theta}}{\boldsymbol{\pi}_{i\theta^*}}\right]\right), \nonumber \\
& = & \log(1) = 0,
\end{eqnarray}
with equality only when $\boldsymbol{\pi}_{i\theta}=\boldsymbol{\pi}_{i\theta^*}$, $\forall i\in N$. Then, we prove that the CA divergence is upper bounded by the average Kullback-Leibler divergence using Jensen's inequality, i.e.,
\begin{eqnarray}
\frac{1}{m}\sum_{i=1}^m D_{KL}(\boldsymbol{\pi}_{i\theta^*}||\boldsymbol{\pi}_{i\theta}) & = & -E^{\theta^*}\left[\frac{1}{m}\sum_{i=1}^m \log \left( \frac{\boldsymbol{\pi}_{i\theta}}{\boldsymbol{\pi}_{i\theta^*}} \right) \right] \nonumber \\
& \ge & -E^{\theta^*}\left[\log \left(\frac{1}{m}\sum_{i=1}^m \frac{\boldsymbol{\pi}_{i\theta}}{\boldsymbol{\pi}_{i\theta^*}} \right) \right] \nonumber \\
&=& D_{CA}(\boldsymbol{\Pi}_{\theta^*}||\boldsymbol{\Pi}_{\theta}).
\end{eqnarray}
with equality only when $\boldsymbol{\pi}_{i\theta}=\boldsymbol{\pi}_{i\theta^*}$, $\forall i\in \mathbf{M}$.
\end{proof}

These results indicate that the DeGroot-style update rule learns that a hypothesis is not the ground truth at a slower rate than the log-linear update rule (\ref{eq:main_algo}). Additionally, we found (through empirical evaluation) that the DeGroot belief for uncertain likelihood models reach a consensus and converge to finite value as the simulations in Section~\ref{sec:SIM} indicates.  This is because the uncertain likelihood ratio update functions $\ell_{i\theta}$ are converging to one. For the certain likelihood condition, the DeGroot rule allows for learning for a fully connected network, as shown in Corollaries~\ref{cor:DG_dog} and \ref{cor:DG_dog2}. Actually, the DeGroot-style rule is able to do this for any network satisfying Assumption~\ref{assum:graph} as indicated next. 

\begin{theorem} \label{thm:DG_dogmatic}
    Let Assumptions \ref{assum:graph} and \ref{assum:inital_beliefs} hold. Then, the beliefs generated by the update rule (\ref{eq:DG_algo}) with infinite evidence converge in probability to:
    \begin{equation}
    \lim_{t\rightarrow \infty} \mu_{it}(\theta) = 0 \ \text{if } \exists j\in\mathbf{M} \ \text{s.t.} \ \boldsymbol{\pi}_{j\theta}\ne\boldsymbol{\pi}_{j\theta^*}.
    \end{equation}
\end{theorem}

\begin{proof}
    The beliefs at time $t$ can be expressed in matrix-vector form as
\begin{eqnarray*}
\boldsymbol{\mu}_t(\theta) &=& \mathbf{L}_\theta(\boldsymbol{\omega}_{t})\mathbf{A}\cdots \mathbf{L}_\theta(\boldsymbol{\omega}_{2})\mathbf{A}\mathbf{L}_\theta(\boldsymbol{\omega}_{1})\mathbf{A}\boldsymbol{\mu}_0(\theta)\\
& =& \prod_{\tau=T+1}^t \left( \mathbf{L}_\theta(\boldsymbol{\omega}_{\tau})\mathbf{A} \right) \boldsymbol{\mu}_T(\theta),
\end{eqnarray*}
where $\mathbf{L}_\theta(\boldsymbol{\omega}_\tau)=  diag(\ell_{1\theta}(\mathbf{n}_{1\tau-1},\omega_{1\tau}|\mathbf{r}_{1\theta}),...,$ $\ell_{m\theta}(\mathbf{n}_{m\tau-1},\omega_{m\tau}|\mathbf{r}_{m\theta}))$, the initial belief $\boldsymbol{\mu}_0(\theta) = \mathbf{1}$ and
\begin{equation*}
\boldsymbol{\mu}_T(\theta) = \prod_{\tau=1}^T \left( \mathbf{L}_\theta(\boldsymbol{\omega}_{\tau})\mathbf{A} \right) \mathbf{1}.
\end{equation*}
For any finite value of $T$, it is clear that $\boldsymbol{\mu}_T(\theta)$ is finite because it can be bounded by a finite number since the norms $\|\mathbf{L}_\theta(\boldsymbol{\omega}_\tau)\|\leq (T+K-1)$ for $1 \le \tau \le T$ via Lemma~\ref{lem:ell_dog_bound} and $\|\mathbf{A}\|=1$. From Lemma~\ref{lem:update_lim} condition 2, it is known that for any $\epsilon>0$ and $\delta>0$ there exist a finite $T$ such that with probability at least $1-\delta$, $\ell_{i\theta}(\mathbf{n}_{i\tau},\omega_{i\tau}|\mathbf{r}_{i\theta}) \leq \frac{\pi_{i\omega_{i\tau}\theta}}{\pi_{i\omega_{i\tau}\theta^*}} + \epsilon$ and $E^{\theta^*}\left[ \ell_{i\theta}(\mathbf{n}_{i\tau},\omega_{i\tau}|\mathbf{r}_{i\theta}) \right] \leq 1+\epsilon$. Let $E^{\theta^*}_{\chi_\nu}[\cdot]$ represent the expectation over the private signals for specific segments in time so that $\chi_\nu = \{ \omega_{i\tau} |  i \in \mathbf{M}, \tau = T+Z_1+\nu Z_2$ for $Z_1 = 1,\ldots,\nu-1$ and $Z_2 = 0,1,\ldots$  $\}$. Now because all the elements of the $\mathbf{A}$ and $\mathbf{L}_\theta$ matrices are non-negative, for $Z>0$ with probability at least $1-\delta$
\begin{eqnarray*}
& & E^{\theta^*}_{\chi_\nu}[\boldsymbol{\mu}_{T+\nu Z}(\theta)] \nonumber \\ & = & \mathbf{L}_\theta(\boldsymbol{\omega}_{T+\nu Z})\mathbf{A} \prod_{\tau=T+\nu(Z-1)+1}^{T+\nu Z-1}\left( E^{\theta^*}\left[\mathbf{L}_\theta(\boldsymbol{\omega}_{\tau})\right]\mathbf{A}\right) \nonumber \\ & & \cdot \mathbf{L}_\theta(\boldsymbol{\omega}_{T+\nu(Z-1)})\mathbf{A} \prod_{\tau=T+\nu(Z-2)+1}^{T+\nu (Z-1)-1}\left( E^{\theta^*}\left[\mathbf{L}_\theta(\boldsymbol{\omega}_{\tau})\right]\mathbf{A}\right)\\
&& \mathbf{L}_\theta(\boldsymbol{\omega}_{T+\nu(Z-2)})\mathbf{A} \cdots \mathbf{L}_\theta(\boldsymbol{\omega}_{T+\nu})\mathbf{A} \nonumber \\ & & \cdot \prod_{\tau=T++1}^{T+\nu-1}\left( E^{\theta^*}\left[\mathbf{L}_\theta(\boldsymbol{\omega}_{\tau})\right]\mathbf{A}\right) \boldsymbol{\mu}_T(\theta)\\
&\leq& (1+\epsilon)^{Z(\nu-1)} \mathbf{L}_\theta(\boldsymbol{\omega}_{T+\nu Z})\mathbf{A}^\nu \mathbf{L}_\theta(\boldsymbol{\omega}_{T+\nu (Z-1)}) \mathbf{A}^\nu \nonumber \\ & & \cdots \mathbf{L}_\theta(\boldsymbol{\omega}_{T+\nu}) \mathbf{A}^\nu \boldsymbol{\mu}_T(\theta).
\end{eqnarray*}
By Lemma~5 in~\cite{NOU2017}, each element of $\mathbf{A^\nu}$ is bounded above by $[\mathbf{A^\nu}]_{ij} \leq \frac{1}{m} + \sqrt{2}m\lambda^\nu$. Thus,
\begin{eqnarray*}
& & E^{\theta^*}_{\chi_\nu}[\boldsymbol{\mu}_{T+\nu Z}(\theta)] \nonumber \\ & \leq &(1+\epsilon)^{Z(\nu-1)} \left(1+\sqrt{2}m\lambda^\nu \right)^Z \mathbf{L}_\theta(\boldsymbol{\omega}_{T+\nu Z})\nonumber \\ && \frac{1}{m}\mathbf{11}' \mathbf{L}_\theta(\boldsymbol{\omega}_{T+\nu (Z-1)})\frac{1}{m}\mathbf{11}' \cdots  \mathbf{L}_\theta(\boldsymbol{\omega}_{T+\nu})\frac{1}{m}\mathbf{11}' \boldsymbol{\mu}_T(\theta)\\
&=& (1+\epsilon)^{Z(\nu-1)} \left(1+\sqrt{2}m\lambda^\nu \right)^Z \nonumber \\ && \cdot \prod_{z=1}^Z \left(\frac{1}{m} \sum_{i=1}^m \ell_{i\theta}(\omega_{i(T+\nu z)}) \right)  \frac{1}{m} \sum_{i=1}^m \mu_{iT}(\theta) \mathbf{L}_\theta(\boldsymbol{\omega}_{T+\nu}) \mathbf{1}.
\end{eqnarray*}
Since $\ell_{i\theta}(\mathbf{n}_{it-1},\omega_{it}|\mathbf{r}_{i\theta}) = \frac{\pi_{i\omega_{it}\theta}}{\pi_{i\omega_{it}\theta*}}$ as $t\to \infty$, then $\log\left( \frac{1}{m} \sum_{i=1}^m \ell_{i\theta}(\mathbf{n}_{it-1},\omega_{it}|\mathbf{r}_{i\theta})\right) = \log\left( \frac{1}{m} \sum_{i=1}^m \frac{\pi_{i\omega_{it}\theta}}{\pi_{i\omega_{it}\theta*}}\right)$ a.s.. Using the fact that $\log(1+x)\le x$ for $x\ge 0$, it is easy to see that for $T$ sufficiently large, the log expected belief can be bounded with probability at least $1-\delta$ as
\begin{eqnarray*}
\log\left( E^{\theta^*}_{\chi_\nu}[\boldsymbol{\mu}_{T+\nu Z}(\theta)]  \right) \leq Z \left(\nu\epsilon+\sqrt{2}m\lambda^\nu \right) \nonumber \\ + \sum_{z=1}^Z \log\left(  \frac{1}{m} \sum_{i=1}^m \frac{\pi_{i\omega_{i(T+\nu z)}\theta}}{\pi_{i\omega_{i(T+\nu z)}\theta*}}\right)+C,
\end{eqnarray*}
where $C = \log\left(\frac{1}{m} \sum_{i=1}^m \mu_{iT}(\theta)\right)+\log\left(\boldsymbol{\ell}_\theta(\mathbf{n}_{iT+\nu  -1}\omega_{iT+\nu }|\mathbf{r}_{i\theta}) \right)$ is a finite constant.
By the law of large numbers for sufficiently large $Z$,
\begin{eqnarray*}
\log\left( E^{\theta^*}_{\chi_\nu}[\boldsymbol{\mu}_{T+\nu Z}(\theta)]  \right) \\ \leq Z \left((\nu-1)\epsilon+\sqrt{2}m\lambda^\nu - D_{CA}(\boldsymbol{\Pi}_{\theta^*}|| \boldsymbol{\Pi}_{\theta}) \right)+C.
\end{eqnarray*}
Since the centralized average divergence is positive as $\boldsymbol{\pi}_{i\theta} \ne \boldsymbol{\pi}_{i\theta^*}$ for at least one agent $i$, $\epsilon$ and $\nu$ can be chosen such that $\nu\epsilon+\sqrt{2}m\lambda^\nu < D_{CA}(\boldsymbol{\Pi}_{\theta^*}|| \boldsymbol{\Pi}_{\theta}) $ and the bounds diverges to $-\infty$ with probability at least $1-\delta$.  Thus, $\lim_{Z \to \infty} E^{\theta^*}_{\chi_\nu}[\boldsymbol{\mu}_{T+\nu Z}(\theta)] \overset{P}{\to} 0$. Finally, the beliefs are always bounded below by zero, and so convergence of the expectation to zero also implies that $\boldsymbol{\mu}_t(\theta) \overset{P}{\to} \infty$.
\end{proof}

\begin{figure}[t]
    \centering
    \includegraphics[width=0.5\columnwidth]{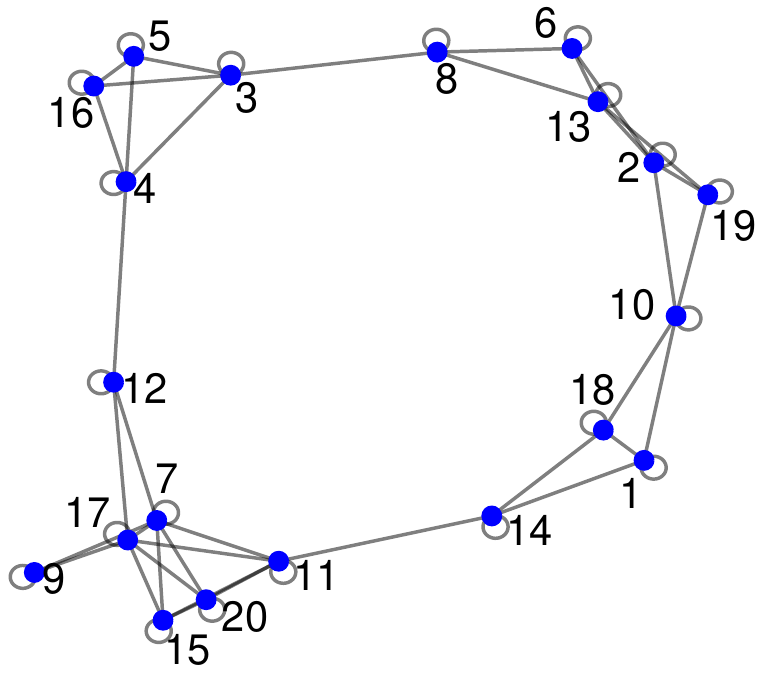} \vspace{-10pt}
    \caption{Example of the network structure considered in the numerical analysis.} \vspace{-12pt}
    \label{fig:graph}
\end{figure}

In summary, DeGroot-style social learning with finite prior evidence does not in general lead to the same beliefs as the centralized uncertain likelihood ratio, unlike the learning rule in~(\ref{eq:main_algo}).  Nevertheless for infinite evidence, learning is still achieved.  For the general case as the uncertain update $\ell_{i\theta}(\mathbf{n}_{it-1},\omega_{it}|\mathbf{r}_{i\theta})$ transitions from a certain-like update to a value of one more slowly as more prior evidence $R_{i\theta}$ is available, more prior evidence leads to a larger chance that beliefs using the DeGroot-style rule will converge to a value greater than one when $\theta=\theta^*$ and a value less than one when $\theta \ne \theta^*$. The experiments in Section~\ref{sec:SIM} empirically show that the interpretation of the beliefs as a uncertain likelihood ratio via Definition~\ref{def:ULRT} is still meaningful even though it is less so than for the social aggregation rule given by (\ref{eq:main_algo}).

\section{Numerical Analysis} \label{sec:SIM}
Next, we present a simulation study of a group of $m=20$ agents applying the proposed algorithms to empirically validate the results. In this study, we considered that the agents are socially connected according to an undirected random geometric graph shown in Figure~\ref{fig:graph}. The weights of the adjacency matrix were constructed using a lazy metropolis matrix \cite{NOU2017} to ensure that the network is doubly stochastic. 

Then, we considered three scenarios based on the amount of prior evidence randomly collected within the following categories: Low, i.e., $R_{i\theta}\in [0, 100]$, High, i.e., $R_{i\theta}\in [1000, 10000]$, and Infinite, i.e., $R_{i\theta}\to \infty$. Within each scenario, each agent randomly selects $R_{i\theta}$ and collects a set of prior evidence for each hypothesis $\theta\in\boldsymbol{\Theta}=\{\theta_1,\theta_2,\theta_3,\theta_4\}$, where the parameters of each hypothesis are shown in Table~\ref{table:theta}. Then, each learning algorithm is simulated for $N=50$ Monte Carlo runs, where the amount of prior evidence, the set of prior evidence, and the measurement sequence is randomly generated during each run. 

\begin{table}[t]
\centering
\caption{Set of hypotheses $\boldsymbol{\Theta}$}
\resizebox{\columnwidth}{!}{\begin{tabular}{c|cccc|}
 & $\theta_1$ & $\theta_2$ & $\theta_3$ & $\theta_4$  \\ \hline
$\boldsymbol{\pi}_{i\theta}$ & $\{0.6, 0.4\}$ & $\{0.55, 0.45\}$ & $\{0.5, 0.5\}$ & $\{0.4, 0.6\}$ \\
$D_{KL}(\boldsymbol{\pi}_{i\theta}||\boldsymbol{\pi}_{i\theta^*})$ & $0$ & $0.0051$ & $0.0204$ & $0.0811$
\end{tabular}} \label{table:theta}
\vspace{-14pt}
\end{table}

\begin{figure*}[t]
	\subfigure[Log-linear $\theta_1$]{
		\centering
		\includegraphics[width=0.24\textwidth]{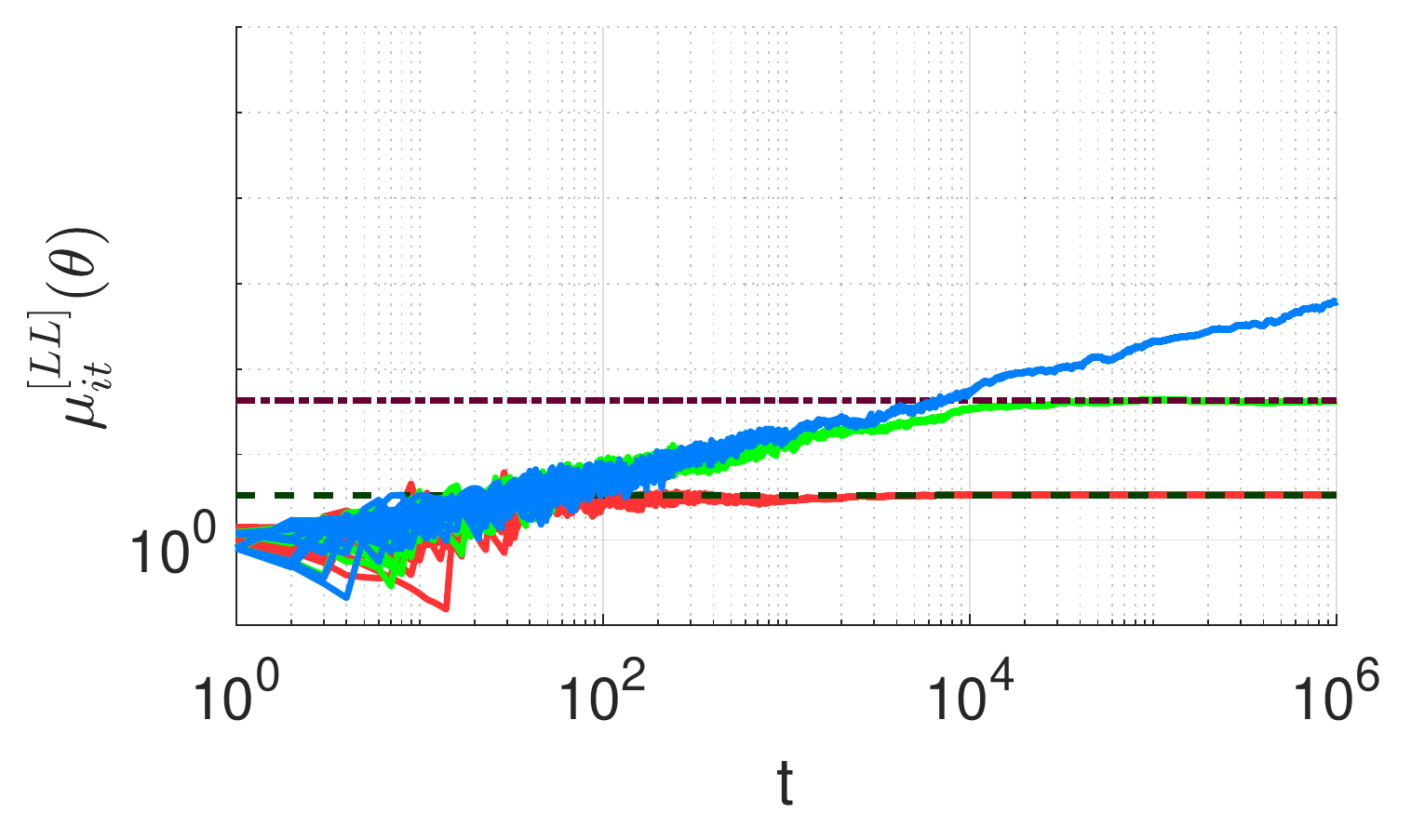}
		\label{fig:ll_t1}
	}%
	\subfigure[Log-linear $\theta_2$]{
		\centering
		\includegraphics[width=0.24\textwidth]{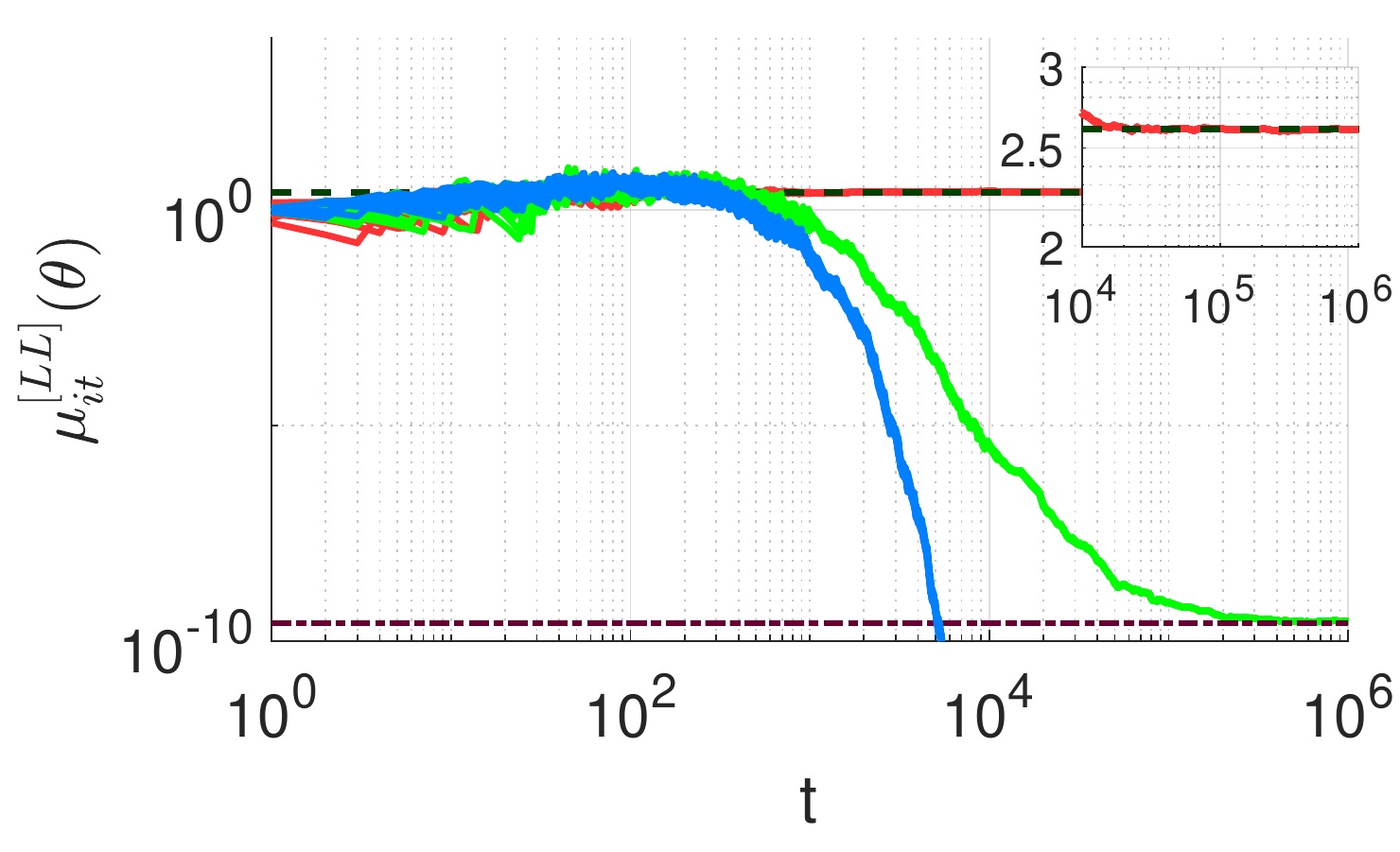}
		\label{fig:ll_t2}
	}%
	\subfigure[Log-linear $\theta_3$]{
	    \centering
		\includegraphics[width=0.24\textwidth]{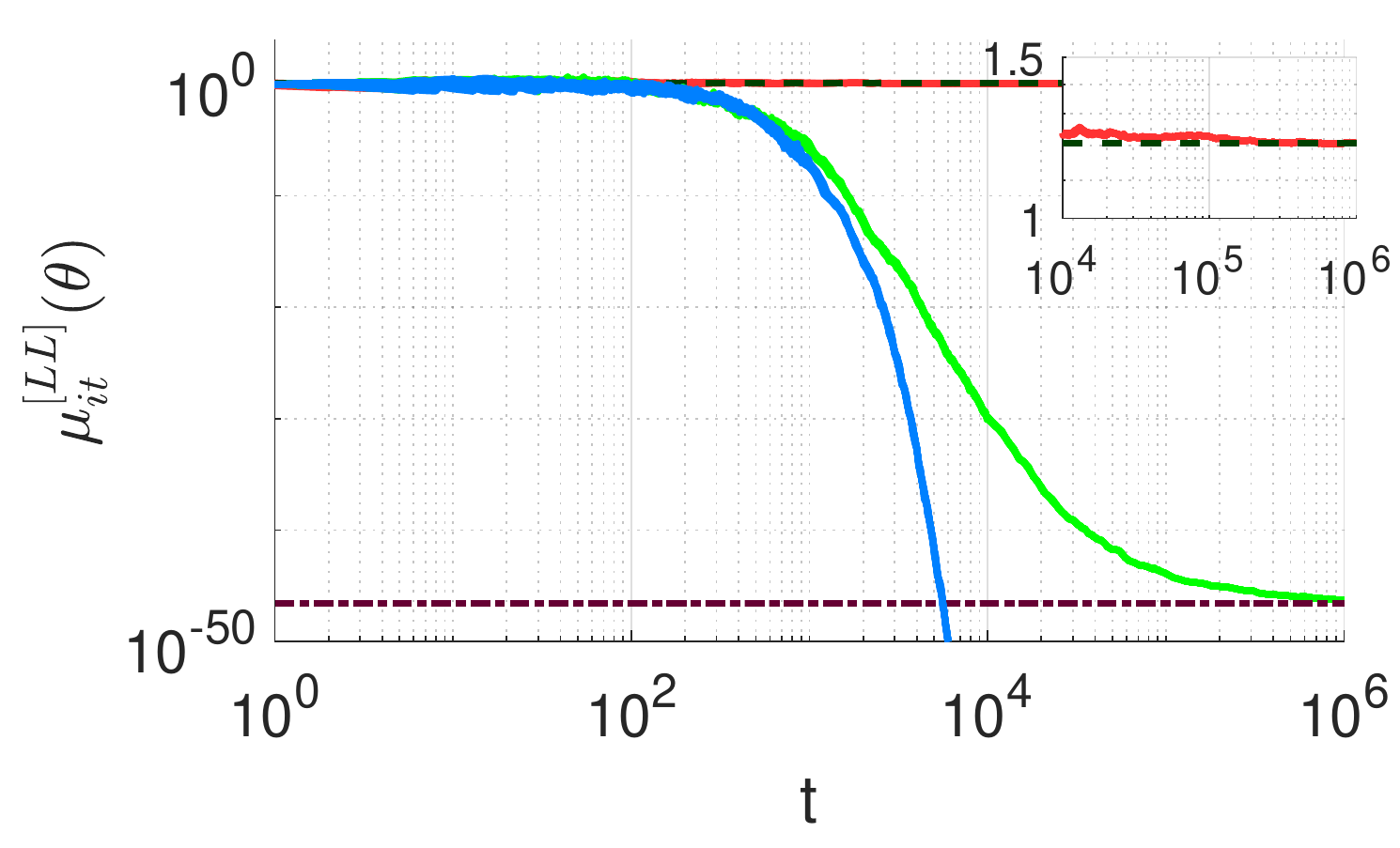}
		\label{fig:ll_t3}
	}%
	\subfigure[Log-linear $\theta_4$]{
	    \centering
		\includegraphics[width=0.24\textwidth]{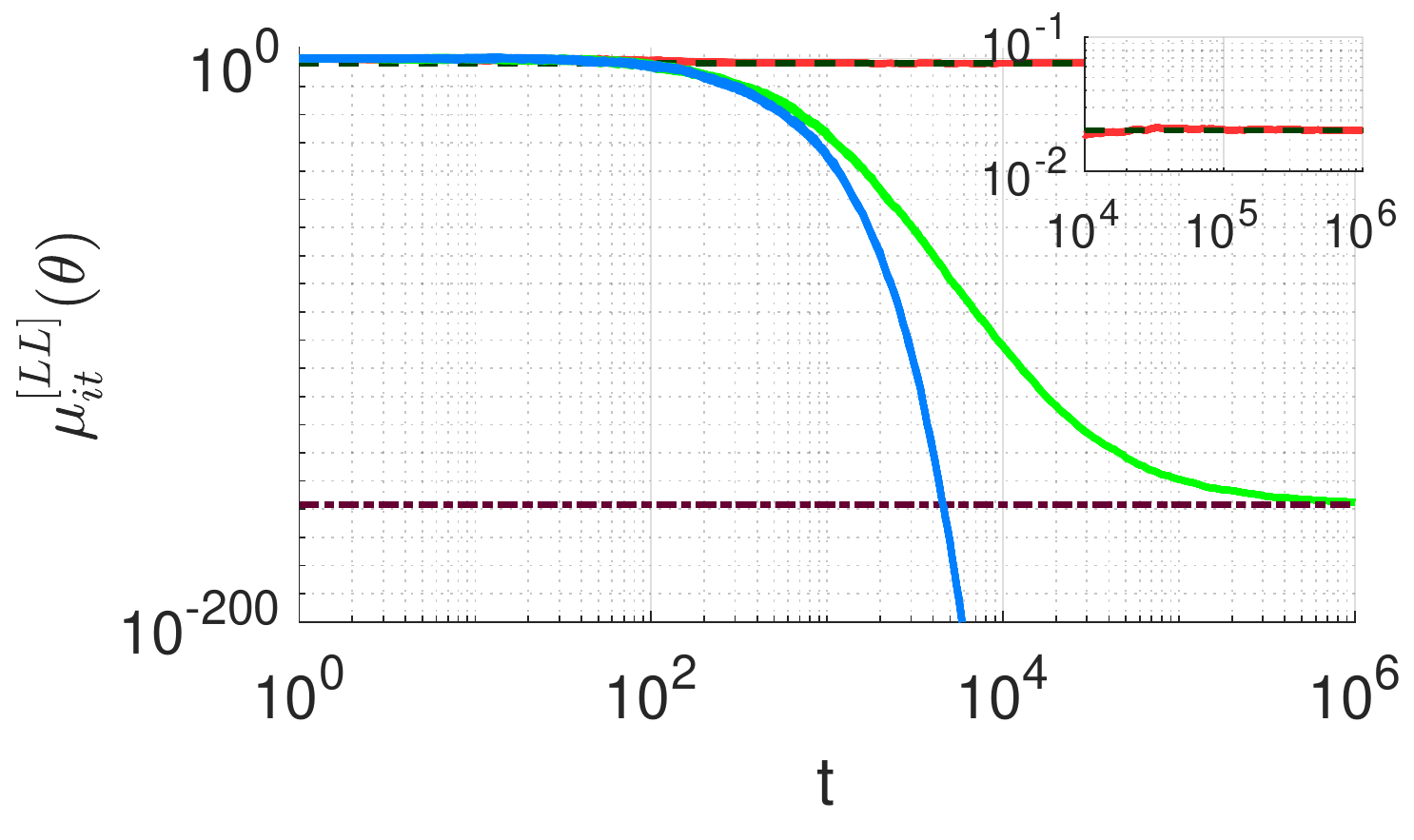}
		\label{fig:ll_t4}
	} \vspace{-6pt}
	
	\subfigure[DeGroot $\theta_1$]{
		\centering
		\includegraphics[width=0.24\textwidth]{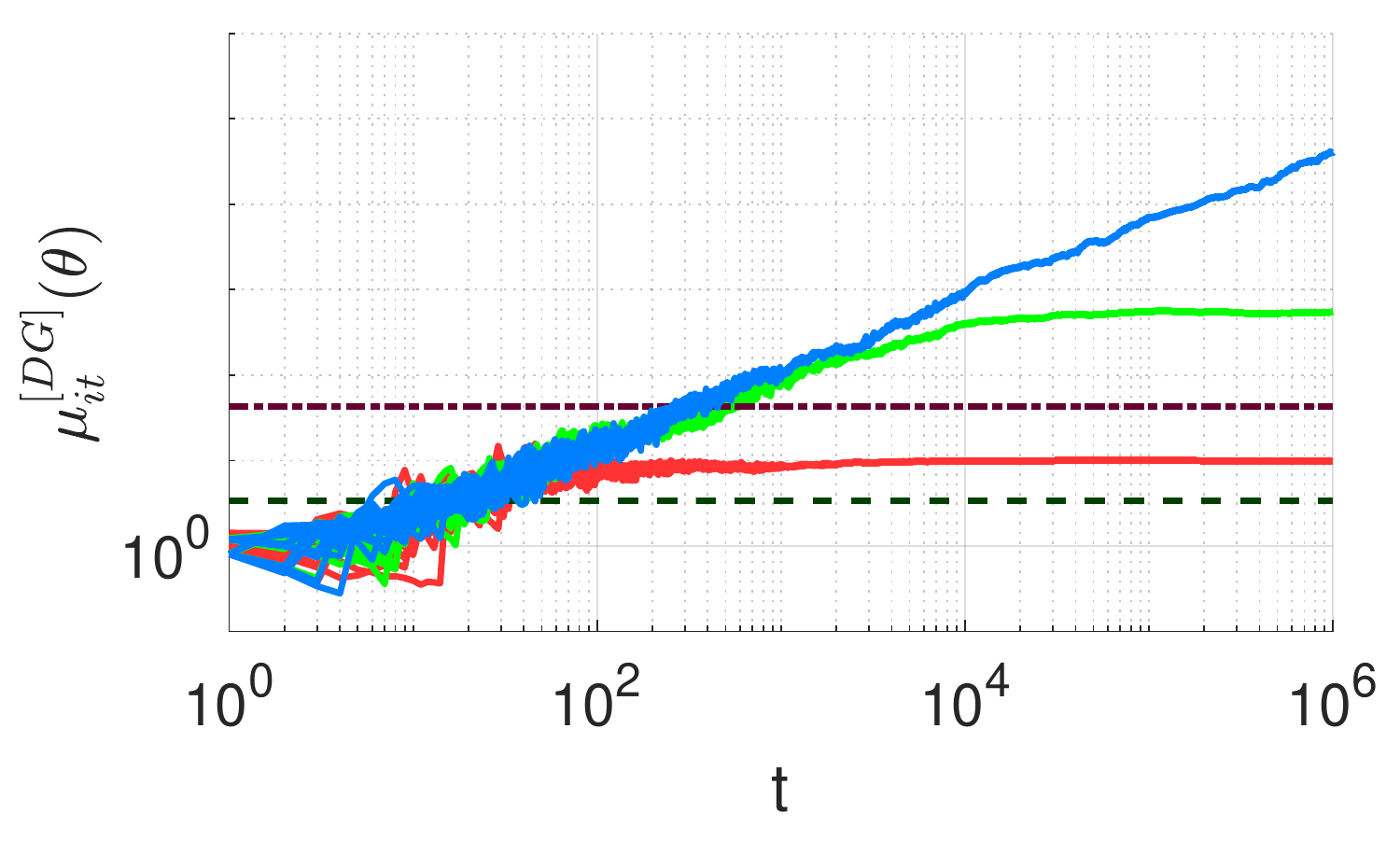}
		\label{fig:dg_t1}
	}%
	\subfigure[DeGroot $\theta_2$]{
		\centering
		\includegraphics[width=0.24\textwidth]{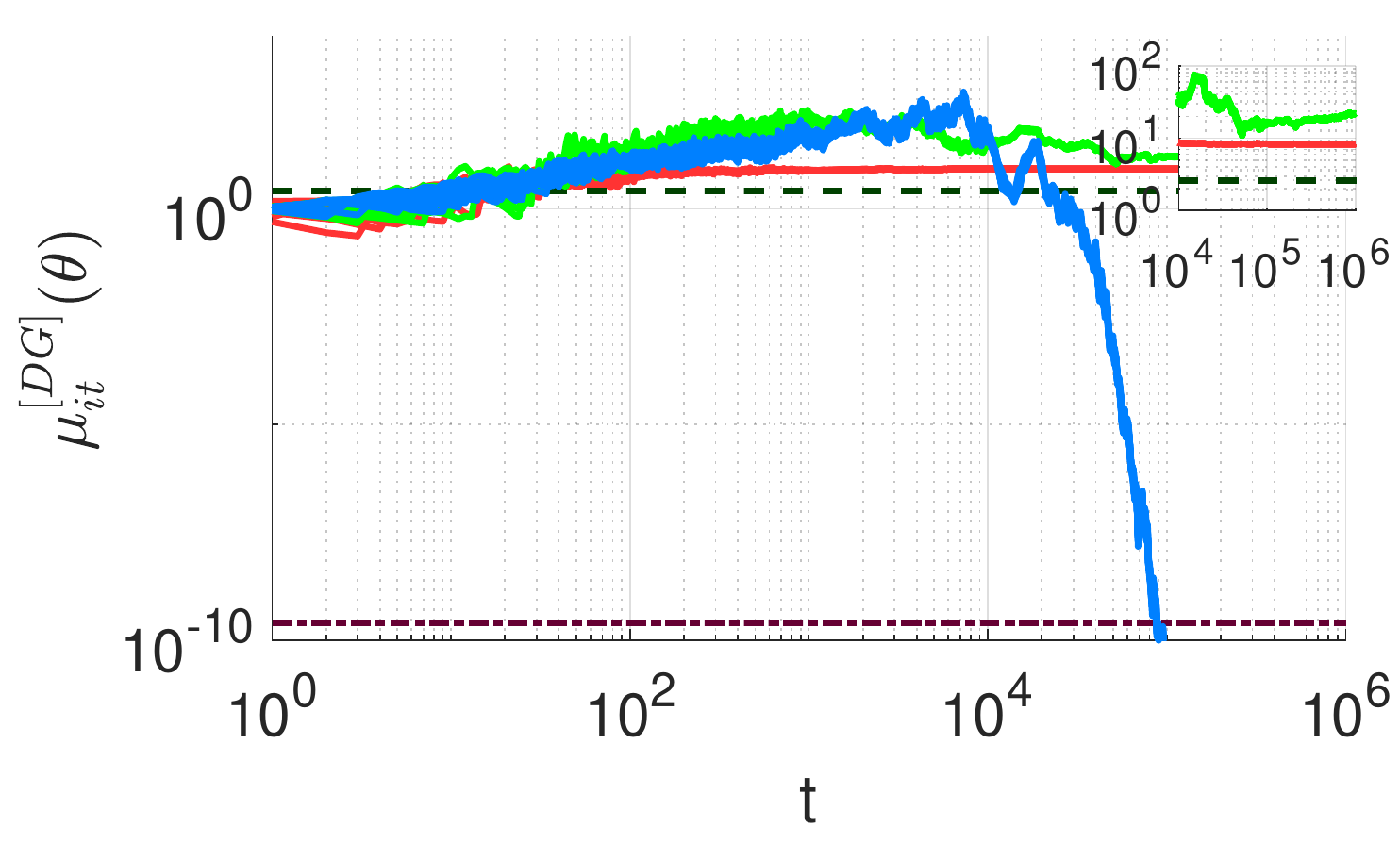}
		\label{fig:dg_t2}
	}%
	\subfigure[DeGroot $\theta_3$]{
	    \centering
		\includegraphics[width=0.24\textwidth]{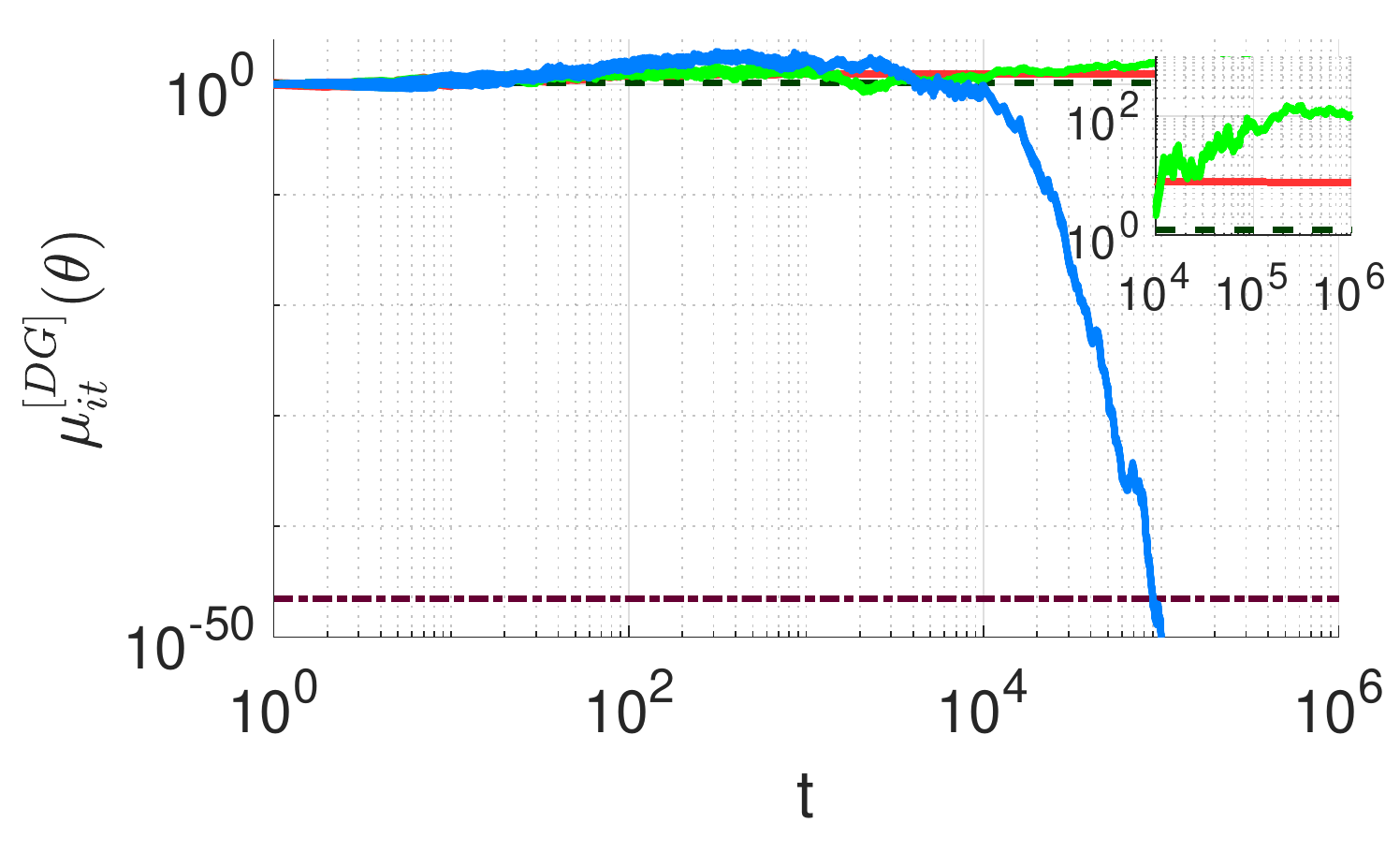}
		\label{fig:dg_t3}
	}%
	\subfigure[DeGroot $\theta_4$]{
	    \centering
		\includegraphics[width=0.24\textwidth]{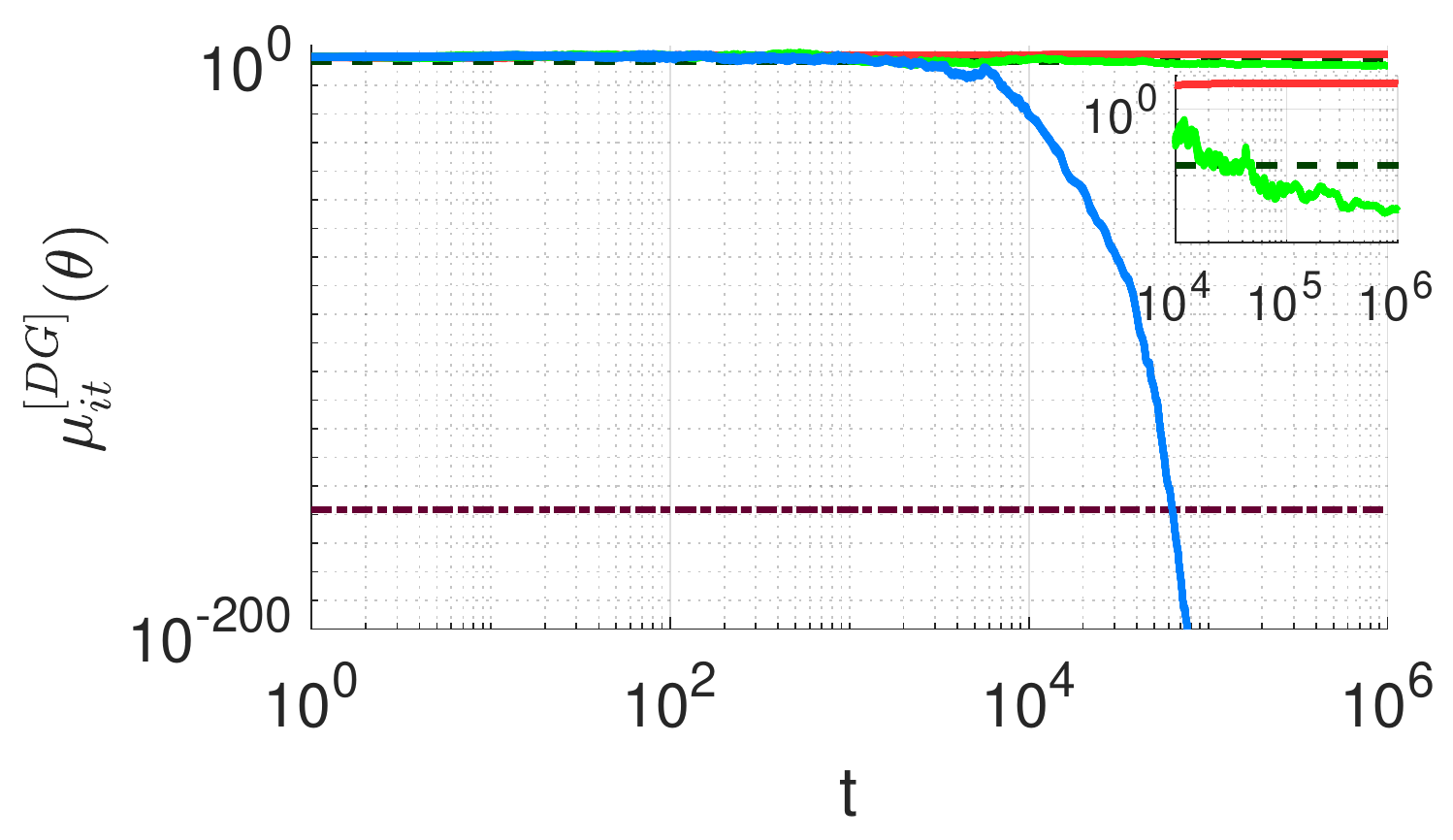}
		\label{fig:dg_t4}
	} \vspace{-6pt}
	
	\centering
	\subfigure{
	    \centering
		\includegraphics[width=0.7\textwidth]{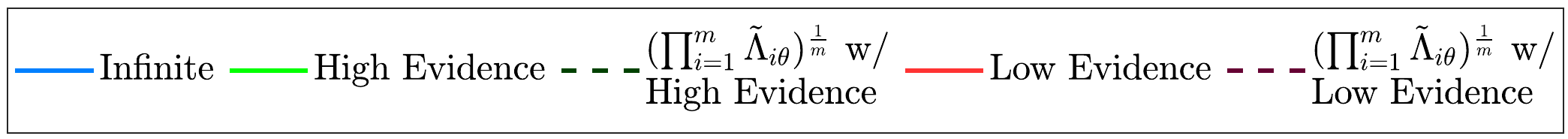}
	} \vspace{-6pt}
	\caption{Belief evolution of the Log-linear (\ref{eq:main_algo}) and DeGroot (\ref{eq:DG_algo}) update rules for hypotheses $\theta_1$, $\theta_2$, $\theta_3$, and $\theta_4$. }\label{fig:mu_graphs} \vspace{-15pt}
\end{figure*}

First, we present the agents' beliefs for both learning rules in Figure~\ref{fig:mu_graphs} for a single Monte Carlo run. These figures show that the amount of prior evidence directly effects the point of convergence of both learning rules. As the evidence increases, the point of convergence increases for $\theta_1=\theta^*$ and decreases for $\theta\ne\theta^*$. Additionally, the log-linear beliefs with finite evidence are converging to $(\prod_{j=1}^m \widetilde{\Lambda}_{j\theta})^\frac{1}{m}$, while the DeGroot beliefs are converging to something larger as stated in Theorem~\ref{thm:ULR_Con} and Lemma~\ref{thm:DG_finite} respectively. This indicates that we could select a threshold that allows for accurate inference with log-linear. However, this is not necessarily the case for the DeGroot model since the beliefs can converge to a value $>1$, as seen for $\theta_2$. The properties of the DeGroot learning rule requires further study as future work. 

Furthermore, these figures show that when the agents are certain, learning occurs as stated in Corollaries~\ref{lem:LL_Dog_Inf}, \ref{lem:LL_dog}, and \ref{cor:DG_dog} and Theorem \ref{thm:DG_dogmatic}. Additionally, we can see that the certain beliefs generated by the DeGroot rule decrease to $0$ at a slower rate than the log-linear beliefs as indicated in Lemma~\ref{thm:comp_rate}. 

\begin{table}
\centering
\caption{Maximum Error Statistics for the uncertain likelihood ratio. }
\begin{tabular}{cc|cc|}
 & & \multicolumn{2}{c|}{$e_{\Lambda_t}(\theta)$}   \\
 \multicolumn{2}{c|}{Time step} & $T=10^3$ & $T=10^6$  \\ \hline 
 \multirow{4}{*}{\rotatebox[origin=c]{90}{Low}} & $\theta_1$ & $2.27^{\diamond}$ & $0.045^{\diamond}$   \\
 & $\theta_2$ & $7.39^{\diamond}$ & $0.060^{\diamond}$     \\
 & $\theta_3$ & $11.15^{\diamond}$ & $0.086^{\diamond}$     \\
 & $\theta_4$ & $100.93^{\diamond}$ & $0.125^{\diamond}$     \\ \hline
  \multirow{4}{*}{\rotatebox[origin=c]{90}{High}} & $\theta_1$ & $267.91^{\diamond}$ & $0.557^{\diamond}$    \\
 & $\theta_2$ & 23.28 & 0.424     \\
 & $\theta_3$ & 7.64 & 1.2e-5    \\
 & $\theta_4$ & 2.2e-12 & 0.00*    \\ \hline
 \multirow{4}{*}{\rotatebox[origin=c]{90}{Infinite}} & $\theta_1$ & n/a$^{\diamond}$ & n/a$^{\diamond}$    \\
 & $\theta_2$ & 25.38 & 0.00*     \\
 & $\theta_3$ & 0.366 & 0.00*    \\
 & $\theta_4$ & 1.3e-16 & 0.00*    \\ 
 \multicolumn{4}{l}{\scriptsize $^{\diamond}$Values are normalized by $\widetilde{\Lambda}_{i\theta}$} \\
 \multicolumn{4}{l}{\scriptsize *Values are less than $10^{-16}$}
\end{tabular} \label{table:error_stats_ulr} \vspace{-15pt}
\end{table}

Next, we studied error statistics to validate the results presented in the previous sections, as seen in Tables~\ref{table:error_stats_ulr}, \ref{table:error_stats_ll}, and \ref{table:error_stats_dg}. First, we consider the maximum error between the uncertain likelihood ratio and the asymptotic uncertain likelihood ratio, i.e., $e_{\Lambda_t}(\theta)=\max_{i\in\mathcal{M},mc\in\{1,...,N\}} |\Lambda_{i\theta}(T,mc)-\widetilde{\Lambda}_{i\theta}(mc)|$, to empirically validate Lemma~\ref{lem:ULR_lim} as seen in Table~\ref{table:error_stats_ulr}. Note that we have normalized the values when the beliefs converge to a value greater than 1, while we do not normalize the values when the beliefs are converging to a value close to $0$ to avoid divide by 0 singularities. 

These results show that as time increases, the error decreases significantly, suggesting that the uncertain likelihood ratio is converging to $\widetilde{\Lambda}_{i\theta}$. Then, as the KL divergence and the amount of evidence increases, the error for hypotheses $\theta\ne\theta^*$ further decreases until the error is $<10^{-16}$, while the error slightly increases for hypotheses $\theta_1$. This is because $\widetilde{\Lambda}_{i\theta_1}$ increases which requires additional time steps for the uncertain likelihood ratio to reach the convergence point. Furthermore, we cannot compute the error for a certain likelihood ratio of $\theta_1$ since $\widetilde{\Lambda}_{i\theta_1}$ is diverging to infinity. However, the median ratio $\Lambda_{i\theta}(T=10^6)/\Lambda_{i\theta}(T=10^3) = 31.55$, indicating that the likelihood ratios are diverging to infinity. 

\begin{table}
\centering
\caption{Maximum Error Statistics for the Log-linear update rule. }
\begin{tabular}{cc|cc|cc|}
 & & \multicolumn{2}{c|}{$e_{\mu_t}^{con}(\theta)$} & \multicolumn{2}{c|}{$e_{\mu_t}^{cen}(\theta)$}  \\
 \multicolumn{2}{c|}{Time step} & $T=10^3$ & $T=10^6$ & $T=10^3$ & $T=10^6$ \\ \hline 
 \multirow{4}{*}{\rotatebox[origin=c]{90}{Low}} & $\theta_1$  & $0.072^{\diamond}$ & $6.1e-5^{\diamond}$ & $0.144^{\triangleright}$ & $3.9e-3^{\triangleright}$  \\
 & $\theta_2$  & $0.086^{\diamond}$ & $1.2e-4^{\diamond}$ & $0.267^{\triangleright}$ & $5.6e-3^{\triangleright}$    \\
 & $\theta_3$  & $0.132^{\diamond}$ & $1.1e-4^{\diamond}$ & $0.403^{\triangleright}$ & $8.8e-3^{\triangleright}$     \\
 & $\theta_4$  & $0.236^{\diamond}$ & $1.6e-4^{\diamond}$ & $1.001^{\triangleright}$ &  $1.7e-2^{\triangleright}$    \\ \hline
  \multirow{4}{*}{\rotatebox[origin=c]{90}{High}} & $\theta_1$ & $0.241^{\diamond}$ & $6.5e-4^{\diamond}$ & $0.802^{\triangleright}$ & $0.132^{\triangleright}$   \\
 & $\theta_2$  & 1.477 & 1.6e-11 & 0.239 & 2.1e-14    \\
 & $\theta_3$  & 6.3e-5 & 0.00* & 7.2e-7 & 0.00*    \\
 & $\theta_4$  & 0.00* & 0.00* & 0.00* & 0.00*   \\ \hline
 \multirow{4}{*}{\rotatebox[origin=c]{90}{Infinite}} & $\theta_1$  & $0.340^{\diamond}$ & $9.3e-3^{\diamond}$ & $n/a^{\triangleright}$ & $n/a^{\triangleright}$   \\
 & $\theta_2$  & 0.504 & 0 & 0.105 & 0    \\
 & $\theta_3$  & 8.7e-7 & 0 & 9.3e-8 & 0    \\
 & $\theta_4$  & 0.00* & 0 & 0.00* & 0    \\ 
 \multicolumn{6}{l}{\scriptsize{$^{\diamond}$Values are normalized by $\bar{\mu}_T(\theta)$}} \\
  \multicolumn{6}{l}{\scriptsize $^{\triangleright}$Values are normalized by $( \prod_{i=1}^m \widetilde{\Lambda}_{i\theta} )^{1/m}$} \\
 \multicolumn{6}{l}{\scriptsize *Values are less than $10^{-16}$}
\end{tabular} \label{table:error_stats_ll} \vspace{-15pt}
\end{table}

The second error statistic shows that the agents converge to a consensus belief, i.e., $e_{\mu_t}^{con}(\theta)=\max_{i\in\mathcal{M}, mc\in\{1,..,N\}} |\mu_{it}(\theta,mc)-\bar{\mu}_T(\theta,mc)|$, where $\bar{\mu}_T(\theta,mc)=\frac{1}{m}\sum_{j=1}^m \mu_{jt}(\theta,mc)$ is the average belief of the agents during the Monte Carlo run $mc$. These results are shown for the log-linear and DeGroot-style learning rules in Tables~\ref{table:error_stats_ll} and \ref{table:error_stats_dg} respectively. Similar to $e_{\Lambda_t}(\theta)$, we normalized the results where the beliefs converge to a value greater than $1$. These tables show that as the number of time steps increases, the error between the agents decreases significantly, thus suggesting that the agents are forming a consensus belief with both rules. Furthermore, it can be seen that the errors between the log-linear and DeGroot beliefs are similar, which suggests that the learning rules are correlated. 

\begin{table}
\centering
\caption{Maximum Error Statistics for the DeGroot-style update rule. }
\begin{tabular}{cc|cc|cc|}
  & & \multicolumn{2}{c|}{$e_{\mu_t}^{con}(\theta)$} & \multicolumn{2}{c|}{$e_{\mu_t}^{cen}(\theta)$}  \\
 \multicolumn{2}{c|}{Time step} & $T=10^3$ & $T=10^6$ & $T=10^3$ & $T=10^6$ \\ \hline 
 \multirow{4}{*}{\rotatebox[origin=c]{90}{Low}} & $\theta_1$  & $0.072^{\diamond}$ & $6.1e-5^{\diamond}$ & $5.497^{\triangleright}$ & $5.638^{\triangleright}$  \\
 & $\theta_2$  & $0.081^{\diamond}$ & $1.2e-4^{\diamond}$ & $5.492^{\triangleright}$ & $5.850^{\triangleright}$    \\
 & $\theta_3$  & $0.138^{\diamond}$ & $1.1e-4^{\diamond}$ & $27.69$ & $25.73$     \\
 & $\theta_4$  & $0.243^{\diamond}$ & $1.6e-4^{\diamond}$ & $25.73$ &  $25.68$    \\ \hline
  \multirow{4}{*}{\rotatebox[origin=c]{90}{High}} & $\theta_1$ & $0.266^{\diamond}$ & $6.5e-4^{\diamond}$ & $19.80^{\triangleright}$ & $111.52^{\triangleright}$   \\
 & $\theta_2$  & $0.751^{\diamond}$ & $4.7e-3^{\diamond}$ & 694.79 & 1.0e4    \\
 & $\theta_3$  & $1.761^{\diamond}$ & $9.4e-3^{\diamond}$ & 2.4e3 & 687.76    \\
 & $\theta_4$  & 132.43 & 1.8e-5 & 269.21 & 1.5e-3   \\ \hline
 \multirow{4}{*}{\rotatebox[origin=c]{90}{Infinite}} & $\theta_1$  & $0.371^{\diamond}$ & $9.3e-3^{\diamond}$ & $n/a^{\triangleright}$ & $n/a^{\triangleright}$   \\
 & $\theta_2$  & 765.33 & 0.00* & 2.0e3 & 0.00*    \\
 & $\theta_3$  & 1.75e3 & 0.00* & 3.8e3 & 0.00*    \\
 & $\theta_4$  & 36.91 & 0.00* & 74.41 & 0.00*    \\ 
 \multicolumn{6}{l}{\scriptsize{$^{\diamond}$Values are normalized by $\bar{\mu}_T(\theta)$}} \\
  \multicolumn{6}{l}{\scriptsize $^{\triangleright}$Values are normalized by $( \prod_{i=1}^m \widetilde{\Lambda}_{i\theta} )^{1/m}$} \\
 \multicolumn{6}{l}{\scriptsize *Values are less than $10^{-16}$}
\end{tabular} \label{table:error_stats_dg} \vspace{-15pt}
\end{table}

Finally, Tables~\ref{table:error_stats_ll} and \ref{table:error_stats_dg} show the error between the agents' beliefs and the centralized uncertain likelihood ratio, i.e.,  $e_{\mu_t}^{cen}(\theta) = \max_{i\in\mathcal{M}, mc\in\{1,...,N\}} |\mu_{it}(\theta,mc)-(\prod_{j=1}^m \widetilde{\Lambda}_{j\theta}(mc))^\frac{1}{m} |$, to empirically validate Theorem~\ref{thm:ULR_Con} and Lemma~\ref{thm:DG_finite}. Similar to the previous results, we have normalized the values where the beliefs converge to a value greater than $1$. The results for the log-linear rule indicate that the beliefs are converging to the centralized uncertain likelihood ratio, while the DeGroot beliefs are converging to a value much larger. When the agents are certain, both learning rules result in beliefs that are converging to $0$ for hypotheses $\theta\ne \theta^*$. Although we cannot evaluate this result for $\theta_1$, we can see that the median of the ratio of beliefs $\mu_{i10^6}(\theta_1)/\mu_{i10^3}(\theta_1)$ is $33.03$ and $880.90$ for log-linear and DeGroot respectively, indicating that the beliefs are diverging to infinity. 

\section{Conclusion} \label{sec:Conclusion}
This work presents the properties of uncertain models in non-Bayesian social learning theory where a group of agents are collaborating together to identify the unknown ground truth hypothesis. Uncertainty arises in many situations where an agent cannot acquire enough prior evidence about a hypothesis to develop precise statistical models. To accommodate for uncertainty, we derived an approximate statistical model for each hypothesis based on the partial information available to a single agent and studied the convergence properties of a group of agents that compute a belief for each hypothesis using a log-linear update rule. We found that when the agents are uncertain, the group forms a consensus belief, albeit different than traditional social beliefs. However, when the agents are certain, the beliefs generated using our uncertain models allow for learning and achieves results consistent with the literature.

We then found that agents can also learn in the certain condition with a DeGroot-style rule, but cannot quantify the convergence point in the uncertain condition. Furthermore, the beliefs generated using the DeGroot-style rule converge at a rate much slower than the log-linear rule. 

As a future work, we will study the effects of malicious agents where preliminary results are presented in \cite{HULJ2019}. Building on analysis of DeGroot-style rules, we will aim to quantify their convergence point as well as those of other aggregation rules. Additionally, we aim to understand how the uncertain likelihood ratio test trades off type I and II errors as a function of prior evidence. 

\ifCLASSOPTIONcaptionsoff
  \newpage
\fi

\bibliographystyle{IEEEtran}
\bibliography{SL_bib,all_refs}

\begin{thebibliography}{10}
\providecommand{\url}[1]{#1}
\csname url@samestyle\endcsname
\providecommand{\newblock}{\relax}
\providecommand{\bibinfo}[2]{#2}
\providecommand{\BIBentrySTDinterwordspacing}{\spaceskip=0pt\relax}
\providecommand{\BIBentryALTinterwordstretchfactor}{4}
\providecommand{\BIBentryALTinterwordspacing}{\spaceskip=\fontdimen2\font plus
\BIBentryALTinterwordstretchfactor\fontdimen3\font minus
  \fontdimen4\font\relax}
\providecommand{\BIBforeignlanguage}[2]{{%
\expandafter\ifx\csname l@#1\endcsname\relax
\typeout{** WARNING: IEEEtran.bst: No hyphenation pattern has been}%
\typeout{** loaded for the language `#1'. Using the pattern for}%
\typeout{** the default language instead.}%
\else
\language=\csname l@#1\endcsname
\fi
#2}}
\providecommand{\BIBdecl}{\relax}
\BIBdecl

\bibitem{JMST2012}
A.~Jadbabaie, P.~Molavi, A.~Sandroni, and A.~Tahbaz-Salehi, ``Non-{B}ayesian
  social learning,'' \emph{Games and Economic Behavior}, vol.~76, no.~1, pp.
  210--225, 2012.

\bibitem{MTJ18}
P.~Molavi, A.~Tahbaz-Salehi, and A.~Jadbabaie, ``A theory of non-{B}ayesian
  social learning,'' \emph{Econometrica}, vol.~86, no.~2, pp. 445--490, 2018.

\bibitem{GK2003}
D.~Gale and S.~Kariv, ``Bayesian learning in social networks,'' \emph{Games and
  Economic Behavior}, vol.~45, no.~2, pp. 329--346, 2003.

\bibitem{ADLO2011}
D.~Acemoglu, M.~A. Dahleh, I.~Lobel, and A.~Ozdaglar, ``{B}ayesian learning in
  social networks,'' \emph{The Review of Economic Studies}, vol.~78, no.~4, pp.
  1201--1236, 2011.

\bibitem{KT2013}
Y.~Kanoria and O.~Tamuz, ``Tractable {B}ayesian social learning on trees,''
  \emph{IEEE Journal on Selected Areas in Communications}, vol.~31, no.~4, pp.
  756--765, 2013.

\bibitem{RJM2017}
M.~A. Rahimian, A.~Jadbabaie, and E.~Mossel, ``Complexity of {B}ayesian belief
  exchange over a network,'' in \emph{Decision and Control (CDC), 2017 IEEE
  56th Annual Conference on}.\hskip 1em plus 0.5em minus 0.4em\relax IEEE,
  2017, pp. 2611--2616.

\bibitem{J2016}
A.~J{\o}sang, \emph{Subjective logic}.\hskip 1em plus 0.5em minus 0.4em\relax
  Springer, 2016.

\bibitem{R1984}
D.~B. Rubin, ``Bayesianly justifiable and relevant frequency calculations for
  the applied statistician,'' \emph{The Annals of Statistics}, vol.~12, no.~4,
  pp. 1151--1172, 1984.

\bibitem{YYZS2018}
K.~Yuan, B.~Ying, X.~Zhao, and A.~H. Sayed, ``Exact diffusion for distributed
  optimization and learning—part i: Algorithm development,'' \emph{IEEE
  Transactions on Signal Processing}, vol.~67, no.~3, pp. 708--723, 2018.

\bibitem{NO2009}
A.~Nedic and A.~Ozdaglar, ``Distributed subgradient methods for multi-agent
  optimization,'' \emph{IEEE Transactions on Automatic Control}, vol.~54,
  no.~1, p.~48, 2009.

\bibitem{KM2011}
S.~Kar and J.~M. Moura, ``Convergence rate analysis of distributed gossip
  (linear parameter) estimation: Fundamental limits and tradeoffs,'' \emph{IEEE
  Journal of Selected Topics in Signal Processing}, vol.~5, no.~4, pp.
  674--690, 2011.

\bibitem{KMR2012}
S.~Kar, J.~M. Moura, and K.~Ramanan, ``Distributed parameter estimation in
  sensor networks: Nonlinear observation models and imperfect communication,''
  \emph{IEEE Transactions on Information Theory}, vol.~58, no.~6, pp.
  3575--3605, 2012.

\bibitem{LS2007}
C.~G. Lopes and A.~H. Sayed, ``Incremental adaptive strategies over distributed
  networks,'' \emph{IEEE Transactions on Signal Processing}, vol.~55, no.~8,
  pp. 4064--4077, 2007.

\bibitem{CS2012}
J.~Chen and A.~H. Sayed, ``Diffusion adaptation strategies for distributed
  optimization and learning over networks,'' \emph{IEEE Transactions on Signal
  Processing}, vol.~60, no.~8, pp. 4289--4305, 2012.

\bibitem{CS2015p1}
------, ``On the learning behavior of adaptive networks—part i: Transient
  analysis,'' \emph{IEEE Transactions on Information Theory}, vol.~61, no.~6,
  pp. 3487--3517, 2015.

\bibitem{CS2015p2}
------, ``On the learning behavior of adaptive networks—part ii: Performance
  analysis,'' \emph{IEEE Transactions on Information Theory}, vol.~61, no.~6,
  pp. 3518--3548, 2015.

\bibitem{D1974}
M.~H. DeGroot, ``Reaching a consensus,'' \emph{Journal of the American
  Statistical Association}, vol.~69, no. 345, pp. 118--121, 1974.

\bibitem{SJ2013}
S.~Shahrampour and A.~Jadbabaie, ``Exponentially fast parameter estimation in
  networks using distributed dual averaging,'' in \emph{52nd IEEE Conference on
  Decision and Control}.\hskip 1em plus 0.5em minus 0.4em\relax IEEE, 2013, pp.
  6196--6201.

\bibitem{MJRT2013}
P.~Molavi, A.~Jadbabaie, K.~R. Rad, and A.~Tahbaz-Salehi, ``Reaching consensus
  with increasing information,'' \emph{IEEE Journal of Selected Topics in
  Signal Processing}, vol.~7, no.~2, pp. 358--369, 2013.

\bibitem{SYS2017}
H.~Salami, B.~Ying, and A.~H. Sayed, ``Social learning over weakly connected
  graphs,'' \emph{IEEE Transactions on Signal and Information Processing over
  Networks}, vol.~3, no.~2, pp. 222--238, 2017.

\bibitem{RT2010}
K.~R. Rad and A.~Tahbaz-Salehi, ``Distributed parameter estimation in
  networks,'' in \emph{49th IEEE Conference on Decision and Control
  (CDC)}.\hskip 1em plus 0.5em minus 0.4em\relax IEEE, 2010, pp. 5050--5055.

\bibitem{RMJ2014}
M.~A. Rahimian, P.~Molavi, and A.~Jadbabaie, ``(non-) {B}ayesian learning
  without recall,'' in \emph{53rd IEEE Conference on Decision and
  Control}.\hskip 1em plus 0.5em minus 0.4em\relax IEEE, 2014, pp. 5730--5735.

\bibitem{SRJ2015}
S.~Shahrampour, A.~Rakhlin, and A.~Jadbabaie, ``Distributed detection:
  Finite-time analysis and impact of network topology,'' \emph{IEEE
  Transactions on Automatic Control}, vol.~61, no.~11, pp. 3256--3268, 2015.

\bibitem{NOU2015}
A.~Nedi{\'c}, A.~Olshevsky, and C.~A. Uribe, ``Nonasymptotic convergence rates
  for cooperative learning over time-varying directed graphs,'' in
  \emph{American Control Conference (ACC), 2015}.\hskip 1em plus 0.5em minus
  0.4em\relax IEEE, 2015, pp. 5884--5889.

\bibitem{LJS2018}
A.~Lalitha, T.~Javidi, and A.~D. Sarwate, ``Social learning and distributed
  hypothesis testing,'' \emph{IEEE Transactions on Information Theory},
  vol.~64, no.~9, pp. 6161--6179, 2018.

\bibitem{LR2018}
G.~Levy and R.~Razin, ``Information diffusion in networks with the {B}ayesian
  peer influence heuristic,'' \emph{Games and Economic Behavior}, vol. 109, pp.
  262--270, 2018.

\bibitem{NOU2017}
A.~Nedi{\'c}, A.~Olshevsky, and C.~A. Uribe, ``Fast convergence rates for
  distributed non-{B}ayesian learning,'' \emph{IEEE Transactions on Automatic
  Control}, vol.~62, no.~11, pp. 5538--5553, 2017.

\bibitem{NOU2016}
------, ``Network independent rates in distributed learning,'' in \emph{2016
  American Control Conference (ACC)}.\hskip 1em plus 0.5em minus 0.4em\relax
  IEEE, 2016, pp. 1072--1077.

\bibitem{BT2018}
M.~Bhotto and W.~P. Tay, ``Non-{B}ayesian social learning with observation
  reuse and soft switching,'' \emph{ACM Transactions on Sensor Networks
  (TOSN)}, vol.~14, no.~2, p.~14, 2018.

\bibitem{SV2018}
L.~Su and N.~H. Vaidya, ``Defending non-{B}ayesian learning against adversarial
  attacks,'' \emph{Distributed Computing}, pp. 1--13, 2018.

\bibitem{DP2001}
D.~Dubois and H.~Prade, ``Possibility theory, probability theory and
  multiple-valued logics: A clarification,'' \emph{Annals of mathematics and
  Artificial Intelligence}, vol.~32, no. 1-4, pp. 35--66, 2001.

\bibitem{K2005}
G.~J. Klir, \emph{Uncertainty and information: {F}oundations of generalized
  information theory}.\hskip 1em plus 0.5em minus 0.4em\relax John Wiley \&
  Sons, 2005.

\bibitem{DP2012}
D.~Dubois and H.~Prade, ``Possibility theory,'' in \emph{Computational
  complexity}.\hskip 1em plus 0.5em minus 0.4em\relax Springer, 2012, pp.
  2240--2252.

\bibitem{W1996}
P.~Walley, ``Inferences from multinomial data: learning about a bag of
  marbles,'' \emph{Journal of the Royal Statistical Society. Series B
  (Methodological)}, pp. 3--57, 1996.

\bibitem{W1997}
------, ``Statistical inferences based on a second-order possibility
  distribution,'' \emph{International Journal of General System}, vol.~26,
  no.~4, pp. 337--383, 1997.

\bibitem{B2005}
J.-M. Bernard, ``An introduction to the imprecise {D}irichlet model for
  multinomial data,'' \emph{International Journal of Approximate Reasoning},
  vol.~39, no. 2-3, pp. 123--150, 2005.

\bibitem{S1976}
G.~Shafer, \emph{A mathematical theory of evidence}.\hskip 1em plus 0.5em minus
  0.4em\relax Princeton university press, 1976, vol.~42.

\bibitem{SK1994}
P.~Smets and R.~Kennes, ``The transferable belief model,'' \emph{Artificial
  intelligence}, vol.~66, no.~2, pp. 191--234, 1994.

\bibitem{GS1982}
P.~G{\"a}rdenfors and N.-E. Sahlin, ``Unreliable probabilities, risk taking,
  and decision making,'' \emph{Synthese}, vol.~53, no.~3, pp. 361--386, 1982.

\bibitem{C1996}
T.~Ch{\'a}vez, ``Modeling and measuring the effects of vagueness in decision
  models,'' \emph{IEEE Transactions on Systems, Man, and Cybernetics-Part A:
  Systems and Humans}, vol.~26, no.~3, pp. 311--323, 1996.

\bibitem{M1994}
X.-L. Meng, ``Posterior predictive $ p $-values,'' \emph{The Annals of
  Statistics}, vol.~22, no.~3, pp. 1142--1160, 1994.

\bibitem{TGM2009}
F.~Tuyl, R.~Gerlach, K.~Mengersen \emph{et~al.}, ``Posterior predictive
  arguments in favor of the {B}ayes-{L}aplace prior as the consensus prior for
  binomial and multinomial parameters,'' \emph{Bayesian analysis}, vol.~4,
  no.~1, pp. 151--158, 2009.

\bibitem{KW1996}
R.~E. Kass and L.~Wasserman, ``The selection of prior distributions by formal
  rules,'' \emph{Journal of the American Statistical Association}, vol.~91, no.
  435, pp. 1343--1370, 1996.

\bibitem{GC2010}
B.~Gharesifard and J.~Cort{\'e}s, ``When does a digraph admit a doubly
  stochastic adjacency matrix?'' in \emph{Proceedings of the 2010 American
  Control Conference}.\hskip 1em plus 0.5em minus 0.4em\relax IEEE, 2010, pp.
  2440--2445.

\bibitem{NO2016}
A.~Nedi{\'c} and A.~Olshevsky, ``Stochastic gradient-push for strongly convex
  functions on time-varying directed graphs,'' \emph{IEEE Transactions on
  Automatic Control}, vol.~61, no.~12, pp. 3936--3947, 2016.

\bibitem{J2001}
A.~J{\o}sang, ``A logic for uncertain probabilities,'' \emph{International
  Journal of Uncertainty, Fuzziness and Knowledge-Based Systems}, vol.~9,
  no.~03, pp. 279--311, 2001.

\bibitem{Laforgia12}
A.~Laforgia and P.~Natalini, ``On the asymptotic expansion of a ratio of gamma
  functions,'' \emph{Journal of mathematical analysis and applications}, vol.
  389, no.~2, pp. 833--837, 2012.

\bibitem{ram10}
S.~S. Ram, A.~Nedi{\'c}, and V.~V. Veeravalli, ``Distributed stochastic
  subgradient projection algorithms for convex optimization,'' \emph{Journal of
  optimization theory and applications}, vol. 147, no.~3, pp. 516--545, 2010.

\bibitem{HULJ2019}
J.~Hare, C.~Uribe, L.~Kaplan, and A.~Jadbabaie, ``On malicious agents in
  non-{B}ayesian social learning with uncertain models,'' in \emph{ISIF/IEEE
  International Conference on Information Fusion}, 2019.

\end{thebibliography}

\newpage

\appendices
\clearpage
\newpage
\end{document}